\documentclass{article}

\usepackage{microtype}
\usepackage{graphicx}
\usepackage{subfigure}
\usepackage{booktabs} 

\usepackage{hyperref}

\usepackage[accepted]{icml2023}

\usepackage{amsmath}
\usepackage{amssymb}
\usepackage{mathtools}
\usepackage{amsthm}

\usepackage{thmtools, thm-restate}
\usepackage{physics}
\usepackage{amsfonts}
\usepackage{makecell}
\usepackage{multirow}

\DeclareMathOperator*{\argmax}{arg\,max}
\DeclareMathOperator*{\argmin}{arg\,min}

\usepackage[capitalize,noabbrev]{cleveref}

\theoremstyle{plain}
\newtheorem{theorem}{Theorem}[section]

\newtheorem{lemma}[theorem]{Lemma}

\theoremstyle{definition}
\newtheorem{definition}[theorem]{Definition}
\newtheorem{assumption}[theorem]{Assumption}
\theoremstyle{remark}
\newtheorem{remark}[theorem]{Remark}

\usepackage[textsize=tiny]{todonotes}

\icmltitlerunning{Policy Regularization with Dataset Constraint for Offline Reinforcement Learning}

\begin{document}

\twocolumn[
\icmltitle{Policy Regularization with Dataset Constraint for Offline Reinforcement Learning}



\icmlsetsymbol{equal}{*}

\begin{icmlauthorlist}
\icmlauthor{Yuhang Ran}{equal,nju}
\icmlauthor{Yi-Chen Li}{equal,nju}
\icmlauthor{Fuxiang Zhang}{nju,polixir}
\icmlauthor{Zongzhang Zhang}{nju}
\icmlauthor{Yang Yu}{nju,polixir}
\end{icmlauthorlist}

\icmlaffiliation{nju}{National Key Laboratory for Novel Software Technology, Nanjing University}
\icmlaffiliation{polixir}{Polixir Technologies}

\icmlcorrespondingauthor{Zongzhang Zhang}{zzzhang@nju.edu.cn}

\icmlkeywords{Offline Reinforcement Learning, Policy Regularization, Dataset Constraint}


\vskip 0.3in
]



\printAffiliationsAndNotice{\icmlEqualContribution} 

\begin{abstract}
    We consider the problem of learning the best possible policy from a fixed dataset, known as offline Reinforcement Learning (RL). A common taxonomy of existing offline RL works is policy regularization, which typically constrains the learned policy by distribution or support of the behavior policy. However, distribution and support constraints are overly conservative since they both force the policy to choose similar actions as the behavior policy when considering particular states. It will limit the learned policy's performance, especially when the behavior policy is sub-optimal. In this paper, we find that regularizing the policy towards the nearest state-action pair can be more effective and thus propose {\bf P}olicy {\bf R}egularization with {\bf D}ataset {\bf C}onstraint (PRDC). When updating the policy in a given state, PRDC searches the offline dataset for the nearest state-action sample and then restricts the policy with the action of this sample. Unlike previous works, PRDC can guide the policy with better behaviors from the dataset, allowing it to choose actions that do not appear in the dataset along with the given state. It is a softer constraint but still keeps enough conservatism from out-of-distribution actions. Empirical evidence and theoretical analysis show that PRDC can alleviate offline RL's fundamentally challenging value overestimation issue with a bounded performance gap. Moreover, on a set of locomotion and navigation tasks, PRDC achieves state-of-the-art performance compared with existing methods. Code is available at \url{https://github.com/LAMDA-RL/PRDC}.
\end{abstract}

\section{Introduction}

Online Reinforcement Learning (RL) has shown remarkable success in a variety of domains such as games \cite{alphagozero}, robotics \cite{rl_robotics}, and recommendation systems \cite{rl_recommendation}. However, learning an optimal policy online demands continual and possibly huge environmental interactions because of \textit{trial-and-error} \cite{rlbook}. For expense or safety concerns, this may be impractical in real-world applications. On the other hand, offline RL learns from a fixed, previously collected dataset, thus eliminating the need for additional interactions during training. Due to the promise of turning datasets into powerful decision-making engines, offline RL has attracted significant interest in recent years \cite{offline_levine}.

One of the fundamental challenges of offline RL is value overestimation in Out-Of-Distribution (OOD) actions (see \cref{sec:over_value}). According to the methodology of dealing with OOD actions, existing works on offline RL could be roughly categorized into the following two taxonomies \cite{offline_theory}: $(\romannumeral1)$ Pessimistic value-based approaches that learn an underestimated or conservative value to discourage choosing OOD actions \cite{cql, combo, edac, mcq, rorl}. $(\romannumeral2)$ Regularized policy-based approaches that constrain the policy to avoid visiting the states and actions that are less covered by the dataset \cite{bcq, awac, fisherbrc, iql, td3_bc, spot}.

Our work focuses on policy regularization. Generally, previous policy regularization approaches have constrained the learned policy by either the distribution \cite{brac} or the support \cite{bear} of the behavior policy. Considering a particular state, however, distribution and support constraints are overly conservative since they both restrict the policy by actions from the behavior policy. It will limit the performance of the policy, especially when the actions from the behavior policy in the dataset are not optimal for the given state. Nevertheless, there are far more actions in the dataset than in a particular state. A natural question thus arises: \textit{Can we guide the policy by all actions in the dataset rather than the limited ones in a given state?}

It motivates us to propose a new approach on policy regularization. When updating the policy in a particular state, our method will search the dataset for the nearest neighbor of the state-action pair, where the action comes from the policy's prediction. Then we will constrain the policy toward the action of the nearest neighbor. This novel constraint can be interpreted as minimizing the \textit{point-to-set} distance between the state-action pair and the dataset. Thus it is neither a distribution constraint nor a support constraint but a \textit{dataset constraint} method.

One benefit of our proposed dataset constraint is that it can relieve excessive pessimism from sub-optimal behaviors of the behavior policy, allowing the policy to choose better actions that do not appear in the dataset along with the given state but still keeping sufficient conservatism from OOD actions. We name our method \textbf{P}olicy \textbf{R}egularization with \textbf{D}ataset \textbf{C}onstraint (PRDC). PRDC can be combined with any actor-critic algorithm, and we instantiate a practical algorithm upon TD3 \cite{td3} with a highly efficient implementation (\cref{sec:method}). Empirical evidence and theoretical analysis show that PRDC is able to effectively alleviate the fundamentally challenging value overestimation issue of offline RL with a bounded performance gap (\cref{sec:method}). On the Gym and AntMaze tasks from D4RL \cite{d4rl}, PRDC achieves state-of-the-art performance compared with previous methods (\cref{sec:exp}).

\section{Preliminaries}

This section will briefly introduce the background, problem setting, and some notations.

\subsection{Reinforcement Learning}
We consider the infinite-horizon Markov Decision Process (MDP) defined by a tuple $( \mathcal{S},\mathcal{A}, p_0, \mathcal{P}, r, \gamma)$, where $\mathcal{S}$ is the state space, $\mathcal{A}$ is the action space, $p_0$ is the initial state distribution, $\mathcal{P}: {\mathcal{S} \times \mathcal{A}} \to \Delta_\mathcal{S}$ is the transition function\footnote{We use $\Delta_X$ to denote the set of distributions over $X$.}, $r: \mathcal{S} \to \mathbb{R}$ is the reward function, and $\gamma \in [0,1)$ is a discount factor. This paper considers  \textit{deterministic} policies and continuous state and action spaces. We assume $\forall (s,a) \in \mathcal{S} \times \mathcal{A}$, $|r(s)| \le R_{\max}$. 

Given the MDP and the agent's policy $\pi: \mathcal{S} \to \mathcal{A}$, the whole decision process runs as follows: At time step $t \in \mathbb{N}$, the agent perceives the environment state $s_t$; then decides to take action $a_t = \pi(s_t)$, resulting in the environment to transit to the next state $s_{t+1}$ and return a reward $r(s_t)$ to the agent. Let $J(\pi)$ be the expected discounted reward of $\pi$, 
\begin{equation}\label{equ:j}
    J(\pi) = \mathbb{E}\left[\sum_{t=0}^\infty \gamma^t r(s_t)\right],
\end{equation}
where the expectation takes over the randomnesses of the initial state distribution $p_0$ and the transition function $\mathcal{P}$. The objective of RL is to learn an optimal policy $\pi^*$ that has maximal expected discounted reward, i.e., 
$$
\pi^* = \argmax_\pi J(\pi).
$$
Let $Q^\pi: \mathcal{S}\times \mathcal{A}\to \mathbb{R}$ be the state-action value function (or Q-function), 
$$
Q^\pi(s,a) = \mathbb{E}\left[\sum_{t=0}^\infty \gamma^t r(s_t)\bigg|s_0=s,a_0=a\right]. 
$$
We further define the occupancy measure $d^\pi: \mathcal{S} \to \mathbb{R}$ of $\pi$,
$$
d^\pi(s') = (1-\gamma)\int_\mathcal{S}\sum_{t=0}^\infty \gamma^t p_0(s)p(s\to s', t,\pi)\dd s,
$$
where $p(s\to s', t, \pi)$ denotes the density at state $s'$ after taking $t$ steps from state $s$ under policy $\pi$. Then \cref{equ:j} could be reformulated \cite{ddpg_convergence} as
\begin{equation}\label{equ:dperf}
    J(\pi) = \frac{1}{1-\gamma}\mathbb{E}_{s\sim d^\pi(s)}[r(s)].
\end{equation}

\subsection{Offline Reinforcement Learning}

As stated in the above subsection, the classical setting of RL requires interactions with the environment during training. However, interaction is sometimes not allowed, especially when the task demands highly in safety or cost. To this end, offline RL, a.k.a., batch RL or data-driven RL, considers learning in an offline manner. Formally speaking, let $\mathcal{D} = \{(s,a,s',r,d)\}$ denotes the set of transitions collected by a behavior policy $\mu$, where $s$, $a$, $s'$, and $r$ are state, action, next state, and reward, respectively; $d$ is a done flag indicating whether $s'$ is a terminal state\footnote{In infinite-horizon MDPs, we also use $d$ to denote whether $s'$ is an absorbing state since there may be no terminal states.}. The goal of offline RL is to learn the best possible policy from $\mathcal{D}$ without further interactions \cite{batchrl}.

\subsection{Value Overestimation Issue of Offline RL}\label{sec:over_value}

We often use the following \textit{one-step} Temporal Difference (TD) update \cite{rlbook} to approximate $Q^\pi$,
\begin{equation}\label{equ:td}
    \hat{Q}^\pi(s,a) \gets \hat{Q}^\pi(s,a) + \eta \delta_t,
\end{equation}
where $\delta_t = \left[r(s) + \gamma (1-d) \hat{Q}^\pi(s',a') - \hat{Q}^\pi(s,a)\right]$, $a' = \pi(s')$ and $\eta$ is a hyper-parameter controlling the step size. With sufficiently enough samples, $\hat{Q}^\pi$ will converge to $Q^\pi$ \cite{td_converge}. But in offline RL, the dataset $\mathcal{D}$ is limited, with partial coverage of the state-action space. Thus, $(s',a')$ may not exist in $\mathcal{D}$ (a.k.a., distribution shift) because $a'$ is predicted by the learned policy $\pi$, not the behavior policy $\mu$. If $\hat{Q}^\pi(s',a')$ is overestimated, the error will continuously backpropagate to the updates of $\hat{Q}^\pi$, eventually causing $\hat{Q}^\pi$ to have overly large outputs for any input state-action. It is known as the value overestimation issue, with which policy regularization \cite{brac} has been proven to be effective to deal. And we can categorize existing works on policy regularization into distribution constraint and support constraint (\cref{sec:related}).

\section{Our Method}\label{sec:method}

We now introduce our method, \textbf{P}olicy \textbf{R}egularization with \textbf{D}ataset \textbf{C}onstraint (PRDC). First, we will begin by defining our proposed dataset constraint objective. Then, we will instantiate a practical algorithm. Finally, we will give a theoretical analysis of why PRDC works in offline RL and a bound of the performance gap.

\subsection{Dataset Constraint}\label{sec:dc}

\begin{figure}[t]
    \begin{center}
        \includegraphics[width=\columnwidth]{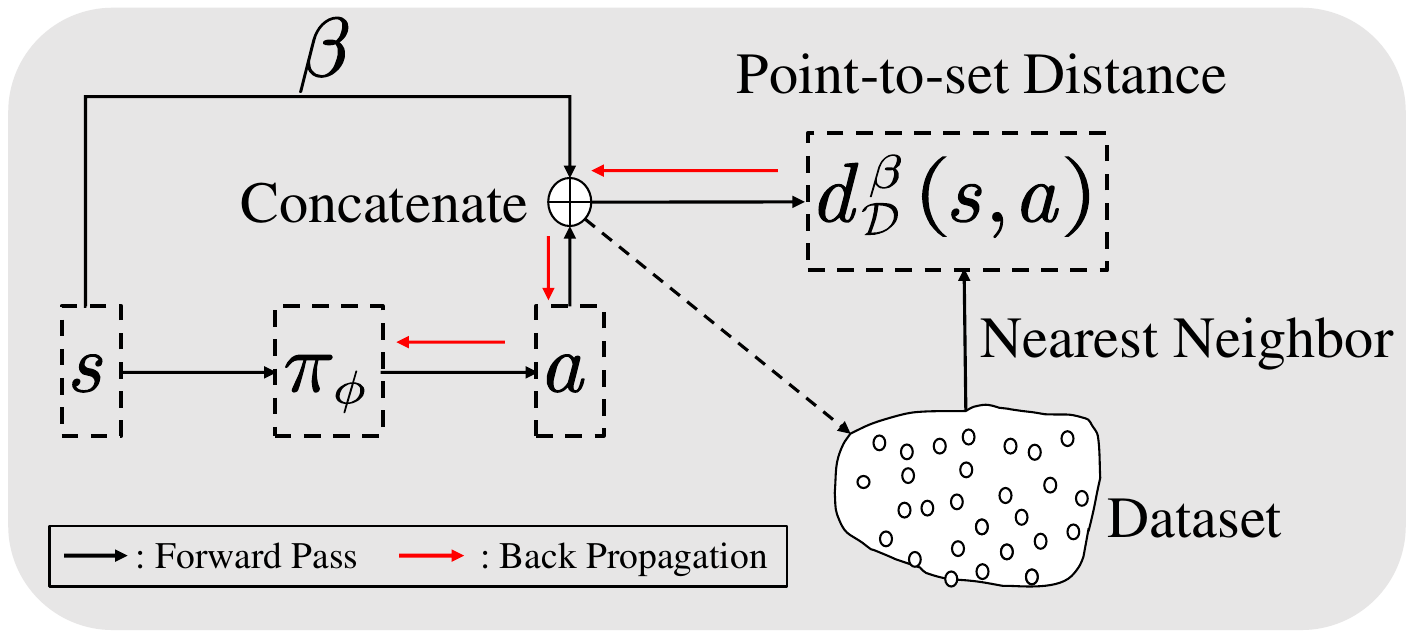}
    \end{center}
    \caption{Illustration of the forward calculation and back propagation of our proposed policy regularization with dataset constraint.}
    \label{fig:method}
\end{figure}

The basic motivation of our \textit{Dataset Constraint (DC)} is to allow the policy $\pi$ to choose optimal actions from all actions in the offline dataset $\mathcal{D}$. Since either distribution constraint or support constraint regularizes $\pi$ by only selecting actions from the same state in the dataset, DC empowers a better generalization ability on $\pi$. However, as claimed in \cref{sec:over_value}, we still have to impose enough conservatism on $\pi$ to avoid the value overestimation issue. This trade-off needs to be carefully balanced.

\begin{definition}[Point-to-set distance]\label{def:psd}
Given the offline dataset $\mathcal{D}$, for any state-action pair $(s, a) \in \mathcal{S}\times \mathcal{A}$, we define its point-to-set distance to $\mathcal{D}$ as
$$
    d^\beta_\mathcal{D}(s,a) = \min_{(\hat{s},\hat{a})\in\mathcal{D}}\|(\beta s)\oplus a-(\beta \hat{s})\oplus \hat{a}\|,
$$
where $\oplus$ denotes the vector concatenation operation and $\beta$ is a hyper-parameter trading off the differences in $s$ and $a$.
\end{definition}

Based on \cref{def:psd}, we give the following objective:
\begin{equation}\label{equ:reg}
    \min_\phi \mathcal{L}_\text{DC}(\phi) := \mathbb{E}_{s\sim \mathcal{D}}\bqty{d^\beta_\mathcal{D}\left(s, \pi_\phi(s)\right)},
\end{equation}
where $\phi$ denotes the learnable parameters of the policy $\pi_\phi$. \cref{fig:method} illustrates the forward calculation and back propagation when using stochastic gradient descent \cite{gd} to optimize \cref{equ:reg}, where we can see that DC will regularize $\pi_\phi$ by minimizing the distance between $(s, \pi_\phi(s))$ and its nearest neighbor in $\mathcal{D}$. Therefore, DC may retrieve a state-action pair $(\hat{s}, \hat{a})$, where $\hat{s}$ is not equal to the given state $s$ and $\hat{a}$ never be along with $s$ in $\mathcal{D}$. This is essentially different from the distribution constraint and the support constraint, where the former forces $\pi(\cdot|s)$ to be similar to $\mu(\cdot|s)$ and the latter requires $\pi(s)$ to be supported by $\mu(\cdot|s)$.

We note that $\beta$ is a key hyper-parameter, controlling the strength of conservatism. Intuitively, when $\beta \to \infty$ or the state space is high-dimensional (e.g., image), \cref{equ:reg} will be dominated by the difference in $s$, and it reduces to behavioral cloning \cite{bc}, which has been proved to be effective on some tasks in TD3+BC \cite{td3_bc}; when $\beta \to 0$, \cref{equ:reg} will ignore the difference in $s$ and it reduces to regularizing $\pi$ such that $\pi(s)$ is close to at least one action in $\mathcal{D}$. However, in-distribution actions are coupled with states. Thus this reduced regularization may not sufficiently constrain the policy from OOD actions. With a proper $\beta$, our state-aware regularization will allow $\pi$ to learn optimal actions from a different state but still maintain enough conservatism.

\subsection{A Practical Algorithm}\label{sec:algo}

\cref{equ:reg} can be combined with any modern actor-critic algorithm, such as TD3 \cite{td3} or SAC \cite{sac}. In this paper, we choose TD3 due to its simplicity and high performance. 

Let $\theta_1, \theta_2, \phi$ be the parameters of TD3's two Q-networks and the policy network, respectively; and $\theta'_1, \theta'_2, \phi'$ denote the corresponding target networks' parameters. TD3 uses the following TD error to update $Q_{\theta_1}$ and $Q_{\theta_2}$:
\begin{equation}\label{equ:q_update}
    \mathcal{L}_{\text{TD}}(\theta_i) = \mathbb{E}_{(s,a,r,s',d) \sim \mathcal{D}}\bqty{\pqty{Q_{\theta_i}(s,a)-y(r,s',d)}^2},
\end{equation}
where $y(r,s',d) = r + \gamma (1-d) \min_i Q_{\theta'_i}(s', a')$, $a' = \texttt{clip}_A(\pi_{\phi'}(s') + \epsilon)$, $\epsilon \sim \texttt{clip}_c(\mathcal{N}(0, \tilde{\sigma}^2))$\footnote{The function $\texttt{clip}_X(\cdot)$ clips its input into $[-X, X]$, $X>0$.}, $i\in \{1,2\}$. $\tilde{\sigma}$ and $c$ are two hyper-parameters for exploration. To update the policy $\pi_\phi$, TD3 uses the loss below:
\begin{equation}\label{equ:pi_udate}
    \mathcal{L}_{\text{TD3}}(\phi) = \mathbb{E}_{s \sim \mathcal{D}, \tilde{a} = \pi_\phi(s)}[-Q_{\theta_1}(s,\tilde{a})].
\end{equation}

Combining \cref{equ:reg} and \cref{equ:pi_udate}, we get the following policy update loss for offline RL:
\begin{equation}\label{equ:offline_pi_update}
    \mathcal{L}_{\text{PRDC}}(\phi) = \lambda \mathcal{L}_{\text{TD3}}(\phi) + \mathcal{L}_{\text{DC}}(\phi).
\end{equation}
Following TD3+BC \cite{td3_bc}, we set $\lambda = \frac{\alpha N}{\sum_{s_i, a_i}\abs{Q(s_i, a_i)}}$ with $\alpha$ a hyper-parameter and $N$ the batch size. The pseudo-code is summarized in Algorithm \ref{alg:algorithm}.

\begin{algorithm}[ht]
    \caption{PRDC}
    \label{alg:algorithm}
    \begin{algorithmic}[1] 
        \STATE {\bfseries Input:} Initial policy parameters $\phi$, Q-function parameters $\theta_1$, $\theta_2$, offline dataset $\mathcal{D}$, hyper-parameters $\alpha, \beta, \tau$. 
        \STATE Set $\theta'_1 \gets \theta_1$, $\theta'_2 \gets \theta_2$, $\phi' \gets \phi$.
        \FOR{step $t = 1$ to $T$}
        \STATE Sample a mini-batch of transitions $\{(s, a, r, s', d)\}$ from $\mathcal{D}$.
        \STATE Update $\theta_i, i\in \{1,2\}$ using gradient descent with \cref{equ:q_update}.
        \STATE Use KD-Tree to find the nearest neighbor in $\mathcal{D}$ of every $(s, \pi_\phi(s))$.
        \STATE Update $\phi$ using gradient descent with \cref{equ:offline_pi_update}. 
        \STATE Update target network with $\theta'_1 \gets \tau \theta_1 + (1-\tau) \theta'_1$, $\theta'_2 \gets \tau \theta_2 + (1-\tau) \theta'_2$, $\phi' \gets \tau \phi + (1-\tau) \phi'$.
        \ENDFOR
    \end{algorithmic}
\end{algorithm}

\begin{remark}[Highly efficient implementation]
Our method requires searching $(s,\pi_\phi(s))$'s nearest neighbor in $\mathcal{D}$, which will be time-consuming if $\mathcal{D}$ is large. To speed up the search, we use KD-Tree \cite{kdtree}. KD-Tree has a time complexity of $\mathcal{O}(M\log |\mathcal{D}|)$ in average for every nearest neighbor retrieval, where $M$ is the feature dimension size. It greatly improves the run time (\cref{sec:run_time}).
\end{remark}

\subsection{Theoretical Analysis}\label{sec:theory}

We will begin with a quantitative analysis of why our method, PRDC, can alleviate the value overestimation issue. Then we will give a performance gap between the learned policy $\pi_\phi$ and the behavior policy $\mu$.

\begin{definition}[Lipschitz function]\label{def:continue}
    A function $f$ from $S \subset \mathbb{R}^m$ into $\mathbb{R}^n$ is called a Lispschitz function if there is a real constant $K\ge 0$ such that
$$
\|f(x)-f(y)\| \le K \|x-y\|,
$$
for all $x, y \in S$. $K$ is called the Lipschitz constant. Unless explicitly stated, we onward use $\|\cdot\|$ to denote the L2 norm.
\end{definition}

We make the following assumptions about the Q-function, the behavior policy $\mu$, and the transition function $\mathcal{P}$.
\begin{assumption}\label{assum:q_continue}
Suppose that the Q-function we learn is a Lipschitz function with $K_Q$ the Lispchitz constant, i.e.,
$$
\|Q(s_1,a_1)-Q(s_2,a_2)\| \le K_Q\|s_1\oplus a_1 -s_2 \oplus a_2\|,
$$
for all $(s_1,a_1), (s_2,a_2)\in\mathcal{S}\times \mathcal{A}$.
\end{assumption}
\begin{assumption}\label{assum:mu_continue}
Suppose that $\mu$ is a Lipschitz function with $K_\mu$ the Lispchitz constant, i.e.,
$$
\|\mu(\cdot|s_1)-\mu(\cdot|s_2)\| \le K_\mu\|s_1 -s_2\|,
$$
for all $s_1, s_2 \in \mathcal{S}$.
\end{assumption}
\begin{assumption}\label{assum:p_continue}
$\forall a_1, a_2 \in \mathcal{A}$, there exists a positive constant $K_\mathcal{P}$ such that
$$
\|\mathcal{P}(s'|s,a_1)-\mathcal{P}(s'|s,a_2)\|\le K_\mathcal{P}\|a_1-a_2\|,
$$
for any $s, s' \in \mathcal{S}$.
\end{assumption}
Since we often use neural networks or linear models to parameterize the value function and policy, \cref{assum:q_continue} and \cref{assum:mu_continue} can be easily satisfied \cite{nn_lips}. \cref{assum:p_continue} is standard in the theoretical studies of RL \cite{classical_theory}.

\begin{restatable}[]{theorem}{thmq}\label{thm:qsmooth}
Let $\max_{s\in\mathcal{S}} d^\beta_\mathcal{D}(s,\pi_\phi(s))\le \epsilon$, which can be achieved by PRDC. Then with \cref{assum:q_continue} and \cref{assum:mu_continue}, we have
\begin{equation}\label{equ:q_smooth}
    \|Q(s,\pi_\phi(s))-Q(s,\mu(s))\|\le \left((K_\mu+2)/\beta + 1\right)K_Q\epsilon,
\end{equation}
for any $s \in \mathcal{S}$.
\end{restatable}

Here we use $\mu(s)$ to denote any action supported by $\mu(\cdot|s)$. The proof is in \cref{sec:proofs}.

Suppose that $\mu$ satisfies \cref{assum:mu_continue}, and both $Q^\pi$ and $\hat{Q}^\pi$ satisfies \cref{assum:q_continue}. Then with \cref{thm:qsmooth}, we have that
\begin{equation}\label{equ:approx2}
    \hat{Q}^\pi(s', \pi(s')) \approx \hat{Q}^\pi(s', \mu(s')),
\end{equation}
\begin{equation}\label{equ:approx3}
    Q^\pi(s', \pi(s')) \approx Q^\pi(s', \mu(s')).
\end{equation}
To alleviate the value overestimation issue, the one-step TD update in \cref{equ:td} requires $\hat{Q}^\pi(s', \pi(s'))$ to be an approximately correct estimate of $Q^\pi(s', \pi(s'))$. \cite{doge} has shown that $\hat{Q}^\pi$ will have low approximation errors on in-distribution samples. Although $(s', \mu(s'))$ may not exist in $\mathcal{D}$, we could still treat it as an in-distribution sample since $s'\in\mathcal{D}$ and $\mathcal{D}$ is constructed by $\mu$. That is, 
\begin{equation}\label{equ:approx1}
    \hat{Q}^\pi(s', \mu(s')) \approx Q^\pi(s', \mu(s')).
\end{equation}
Combining \cref{equ:approx2}, \cref{equ:approx3} and \cref{equ:approx1}, we get $\hat{Q}^\pi(s', \pi(s')) \approx Q^\pi(s', \pi(s'))$. We thus conclude that PRDC can alleviate the value overestimation issue. %

\begin{restatable}[Performance gap of PRDC]{theorem}{perf}\label{thm:perf}
With \cref{assum:mu_continue} and \cref{assum:p_continue}, let $\max_{s\in \mathcal{S}} \abs{\pi^*(s)-\mu(s)}\leq \epsilon_{\rm opt}$ and $\max_{s\in\mathcal{S}} d^\beta_\mathcal{D}(s,\pi_\phi(s))\le \epsilon_\pi$, which can be achieved by PRDC. Then we have
\begin{equation}\label{equ:perf_gap}
    \abs{J(\pi^*)-J(\pi)} \le \frac{CK_{\mathcal{P}}R_{\max}}{1-\gamma}\bqty{(1+\frac{K_\mu}{\beta})\epsilon_\pi + \epsilon_{\rm opt}},
\end{equation}
where $C$ is a positive constant.
\end{restatable}

We defer the proof to \cref{sec:proofs}. \cref{thm:perf} provides \textit{sufficient} conditions for optimality, where we see that the performance gap is inversely related to $\beta$. However, as discussed in \ref{sec:dc}, a small $\beta$ may not keep the policy from OOD actions and thus degrades the performance. With a proper $\beta$, our method can reach higher performance than TD3+BC (corresponding to $\beta \to \infty$). Moreover, for any state $s$, if we assume that there exists another state-action pair $(\tilde{s}, \tilde{a}) \in \mathcal{D}$ such that $\tilde{a}$ is the optimal action for state $s$, then with a properly selected $\beta$, we can get that
\(
    (\tilde{s}, \tilde{a}) = \argmin_{(\hat{s},\hat{a})\in \mathcal{D}} d^\beta_\mathcal{D}(s, \pi(s)),
\)
where we use $(\hat{s}, \hat{a})\in \mathcal{D}$ to denote all the state-action pairs in $\mathcal{D}$. Regularizing $\pi(s)$ toward $\tilde{a}$ will eventually give us an optimal policy. That is, we can obtain a \textbf{zero} performance gap between $\pi^*$ and $\pi$ under the assumption above.

\section{Related Work}\label{sec:related}

\begin{table*}[ht]
    \caption{Average normalized score over the final 10 evaluations and 5 seeds. Scores with the highest mean are highlighted.}
    \label{tab:benchmark}
    \begin{center}
    \resizebox{\textwidth}{!}{\begin{tabular}{lrrrrrrrrr||r@{\,}l} 
    \toprule
                             Task Name & BC & BCQ & BEAR & AWAC & CQL & IQL & TD3+BC & DOGE & SPOT & \multicolumn{2}{c}{\colorbox{lightgray}{PRDC (Ours)}}\\
    \midrule
    halfcheetah-random      & 0.2 & 8.8 & 15.1  & --- &     20.0    &   11.2 &  11.0   &  17.8 &  ---  & \textbf{26.9} & $\pm$\,\,\, 1.0     \\
    hopper-random           & 4.9 & 7.1 & 14.2 & ---  &     8.3     &    7.9  & 8.5    &  21.1 &  ---  &  \textbf{26.8} & $\pm$\,\,\,  9.3   \\
    walker2d-random         & 1.7 & 6.5 & \textbf{10.7} & --- &     8.3 &  5.9   &   1.6 &  0.9 &  ---  &   5.0 &  $\pm$\,\,\, 1.2   \\
    \hline
    halfcheetah-medium     & 42.6 & 47.0 & 41.0 & 43.5  &  44.0  &   47.4  &    48.3   & 45.3  & 58.4  & \textbf{63.5} & $\pm$\,\,\, 0.9  \\
    hopper-medium          & 52.9 & 56.7 & 51.9 & 57.0 &     58.5  &      66.2 &     59.3 & 98.6 & 86.0  &    \textbf{100.3} & $\pm$\,\,\, 0.2  \\
    walker2d-medium        & 75.3 & 72.6 & 80.9 & 72.4     &    72.5    &      78.3     &  83.7 & \textbf{86.8} &  86.4  &   85.2 & $\pm$\,\,\, 0.4  \\
    \hline
    halfcheetah-medium-replay &  36.6 & 40.4 & 29.7 & 40.5 &   45.5     &     44.2       &    44.6 & 42.8  & 52.2  &   \textbf{55.0} & $\pm$\,\,\, 1.1  \\
    hopper-medium-replay      &  18.1 & 53.3 & 37.3 &  37.2 &  95.0    &     94.7     &    60.9  & 76.2 & \textbf{100.2}  &   100.1 & $\pm$\,\,\, 1.6 \\
    walker2d-medium-replay    &  26.0 & 52.1 & 18.5 &  27.0 &  77.2    &     73.8      &    81.8 & 87.3 & 91.6  &  \textbf{92.0} & $\pm$\,\,\, 1.6    \\
    \hline
    halfcheetah-medium-expert &  55.2 & 89.1 & 38.9 &  42.8 &  91.6    &     86.7     &    90.7 & 78.7 & 86.9 &    \textbf{94.5} & $\pm$\,\,\, 0.5   \\
    hopper-medium-expert      &  52.5 & 81.8 & 17.7 &  55.8 &  105.4    &    91.5    &   98.0 & 102.7  & 99.3  &   \textbf{109.2} & $\pm$\,\,\, 4.0    \\
    walker2d-medium-expert    & 107.5 & 109.5 & 95.4  &  74.5 & 108.8  &   109.6       &  110.1 & 110.4 & \textbf{112.0} &   111.2 & $\pm$\,\,\, 0.6   \\
    \hline
    antmaze-umaze   & 65.0  & 78.9 & 73.0 & 56.7 & 84.8  &   88.2      &  91.3 &   97.0 & 93.5 & \textbf{98.8} &  $\pm$\,\,\, 1.0   \\
    antmaze-umaze-diverse   & 55.6  & 61.0 & 61.0 &  49.3 &  43.3    &   66.7       &  54.6 & 63.5 &   40.7   & \textbf{90.0} & $\pm$\,\,\, 6.8   \\
    antmaze-medium-play   & 0.0 & 0.0 & 0.0  & 0.0 &    65.2    &   70.4       &  0.0  & 80.6 & 74.7  & \textbf{82.8} & $\pm$\,\,\, 4.8   \\
    antmaze-medium-diverse   & 0.0 & 0.0 & 8.0  & 0.7 &   54.0   &   74.6      &  0.0  & 77.6 & \textbf{79.1}  & 78.8 & $\pm$\,\,\, 6.9  \\
    antmaze-large-play   & 0.0 & 6.7 & 0.0  & 0.0 &   18.8   &   43.5       &  0.0   & 48.2 & 35.3 & \textbf{54.8} & $\pm$ 10.9   \\
    antmaze-large-diverse   & 0.0 & 2.2 & 0.0  & 1.0 &   31.6   &   45.6      &  0.0  & 36.4 & 36.3  & \textbf{50.0} & $\pm$\,\,\, 5.4   \\
    \bottomrule
    \end{tabular}}
    \end{center}
    \vskip -0.1in
\end{table*}

Policy regularization is a typical way in offline RL to avoid OOD actions. The basic idea is to augment an actor-critic algorithm with a penalty measuring the divergence of the policy from the offline dataset \cite{fisherbrc}. One of the first policy regularization methods in offline RL is BCQ \cite{bcq}. BCQ firstly uses CVAE \cite{cvae} to fit the behavior policy $\mu$ and then learns the policy $\pi$, which has a similar distribution to that of $\mu$. Nevertheless, the fitting error will backpropagate to $\pi$ and eventually affect $\pi$'s performance. TD3+BC \cite{td3_bc} adds a behavioral cloning \cite{bc} term to the policy improvement loss of TD3 \cite{td3}, successfully constraining $\pi$ without an explicit fit of $\mu$. Although simple, TD3+BC achieves competitive results on the Gym-MuJoCo suite of the D4RL benchmark \cite{d4rl}, compared with state-of-the-art methods. TD3+BC shows that even a simple regularization term can achieve superior performance, which greatly inspires the designation of our method.

Both BCQ and TD3+BC constrain $\pi$ to match the distribution of $\mu$. Although effectively keeping the policy from OOD actions, distribution constraint limits the performance of $\pi$ since it cannot distinguish the optimal actions from the poor ones. DOGE \cite{doge} proposes to regularize the policy within the convex hull of dataset via a state-conditioned distance function, which regresses the expectation of actions along with the given state in the offline dataset. So $\pi$ is still limited by the distribution of the behavior policy. To this end, BEAR \cite{bear} uses the maximum mean discrepancy divergence (MMD) \cite{mmd} with a Gaussian kernel as the $f$-divergence to constrain $\pi$. Empirically, they found that when computing MMD over a small number of samples, the sampled MMD between $\mu$ and $\pi$ was similar to the MMD between the supports of $\mu$ and $\pi$. In their experiments, they showed that this support constraint could find optimal policies even when the offline dataset $\mathcal{D}$ was composed of several sub-optimal behaviors. Our method motivates from a quite different perspective. It is neither a distribution constraint method nor a support constraint method. Instead, we constrain $\pi$ with the whole dataset $\mathcal{D}$. Moreover, we want to learn optimal actions from all those in $\mathcal{D}$ instead of just those in a particular state. For a thorough overview of existing offline RL methods and their difference from ours, we recommend referring to \cite{offline_overview}.

Apart from policy regularization, there are also some works associated with dataset constraint. \cite{retrieval_rl} augmented an RL agent with a retrieval process (parameterized as a neural network) that has direct access to a dataset of experiences. They tested the retrieval process in multi-task offline RL settings, showing that it can learn good task representations. However, it is about something other than regularizing the policy from OOD actions. \cite{retrieval_nn} proposed to search the dataset for the current state's nearest neighbor, which would be fed as an additional input into the policy and value networks. They use SCaNN \cite{scann} for fast approximate nearest-neighbor retrieval, while we use KD-Tree to speed up the retrieval for the exact nearest neighbor.

\section{Experiments}\label{sec:exp}

In this section, we will conduct extensive evaluations of the empirical performance of our method, PRDC. We would like to answer the following questions:

$(\romannumeral1)$ \textit{How does PRDC perform on the generally used benchmarks? (\cref{sec:benchmark})}

$(\romannumeral2)$ \textit{Can PRDC learn optimal actions, even when they only come with different states in the dataset? (\cref{sec:generalization})}

$(\romannumeral3)$ \textit{Does PRDC really alleviate the value overestimation issue? (\cref{sec:value_error})}

$(\romannumeral4)$ \textit{How does the additional hyper-parameter $\beta$ influence PRDC's performance? (\cref{sec:beta_ablation})}

Due to the nearest neighbor searching operation in \cref{equ:reg}, PRDC may be time-consuming if not implemented well. Therefore, we present an empirical comparison of the running time of PRDC and existing methods (\cref{sec:run_time}).

\subsection{Main Results on Benchmark}\label{sec:benchmark}

\begin{figure*}[ht]
    \begin{center}
        \subfigure[lineworld-easy]{
            \includegraphics[width=0.42\linewidth]{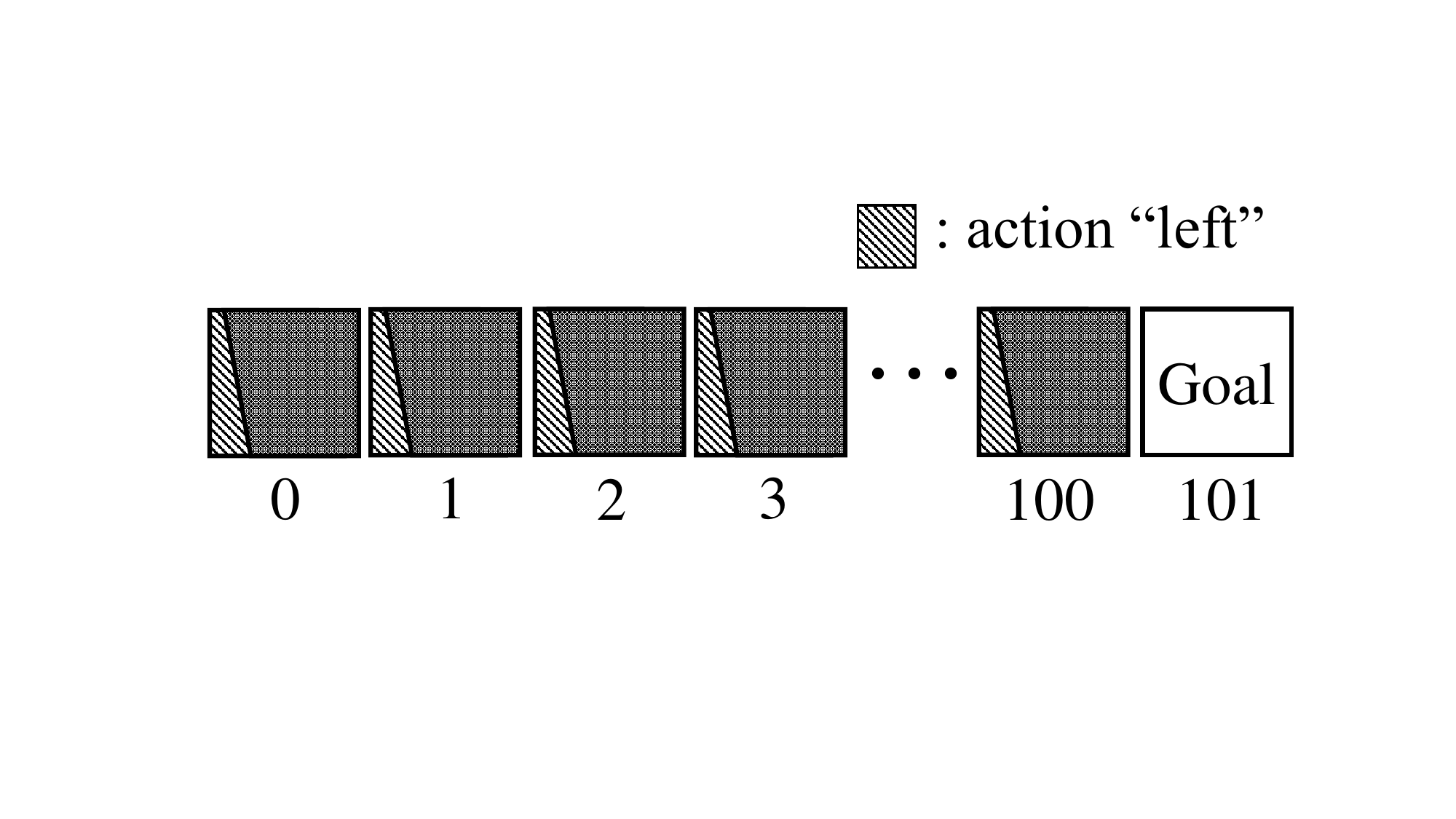}
            \label{fig:easy}
        }
        \subfigure[lineworld-medium]{
            \includegraphics[width=0.42\linewidth]{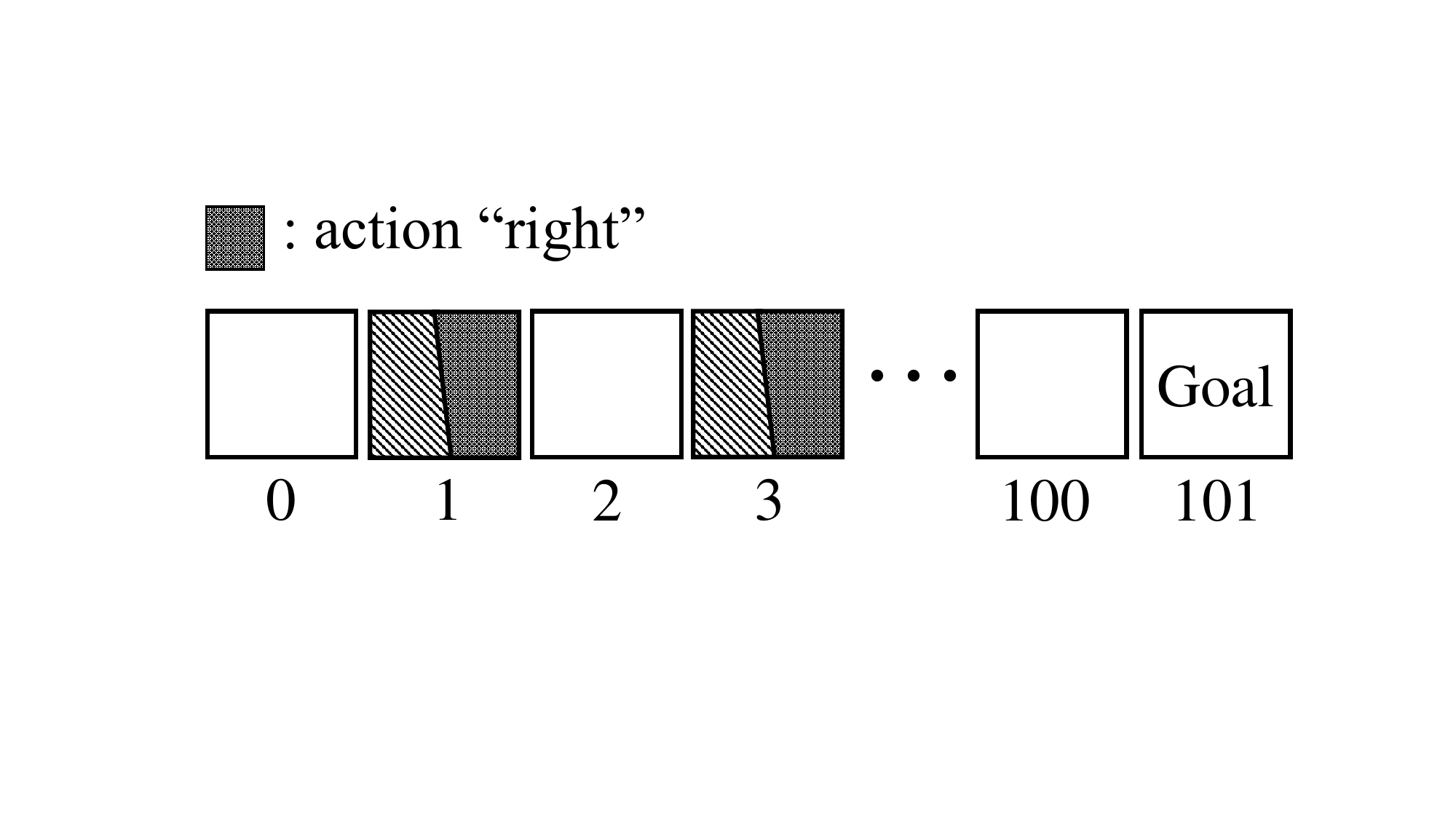}
            \label{fig:medium}
        }
        \subfigure[lineworld-hard]{
            \includegraphics[width=0.42\linewidth]{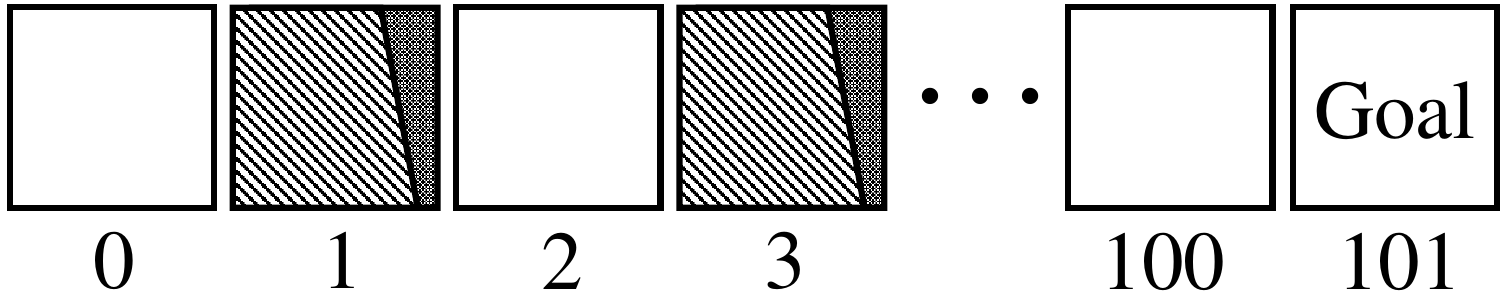}
            \label{fig:hard}
        }
        \subfigure[lineworld-superhard]{
            \includegraphics[width=0.42\linewidth]{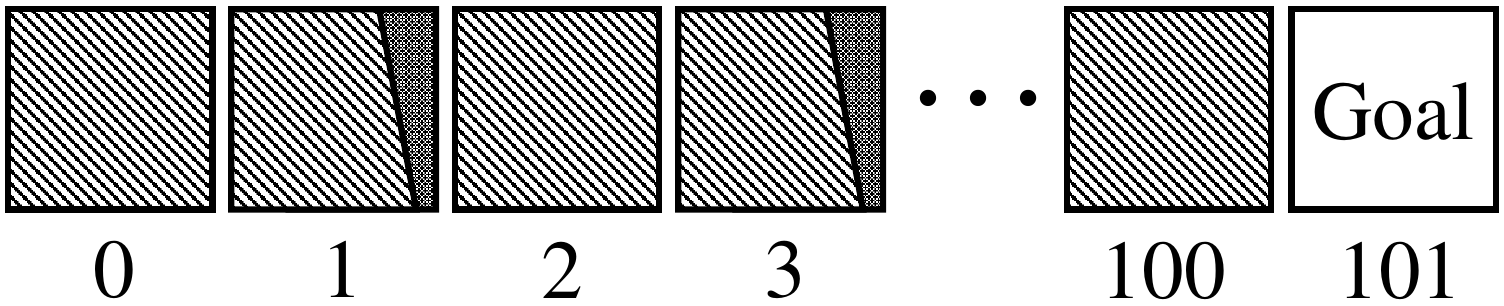}
            \label{fig:superhard}
        }
        \caption{Compositions of four lineworld datasets. In all of the four datasets, the actions are in $\{-1, +1\}$. (a) \textit{lineworld-easy}: the states are in $\{0, 1, 2, 3, \cdots, 100\}$, where at each state, the ratio of action $-1$ and $+1$ is $1:99$. (b) \textit{lineworld-medium}: The states are in $\{1, 3, 5, 7, \cdots, 99\}$, where at each state, the ratio of action $-1$ and $+1$ is $1:1$. (c) \textit{lineworld-hard}: the states are in $\{1, 3, 5, 7, \cdots, 99\}$, where at each state, the ratio of action $-1$ and $+1$ is $99:1$. (d) \textit{lineworld-superhard}: the states are in $\{0, 1, 2, 3, \cdots, 100\}$. If the states are in $\{1, 3, 5, 7, \cdots, 99\}$, the ratio of action $-1$ and $+1$ is $99:1$. If the states are in $\{0, 2, 4, 6, \cdots, 100\}$, all actions are $-1$. In all the four datasets, we collect $100$ samples for every state if there has actions.}
        \label{fig:lineworld}
    \end{center}
\end{figure*}

\begin{figure*}[ht]
    \begin{center}
        \subfigure[BEAR on lineworld-medium]{
            \includegraphics[width=0.31\linewidth]{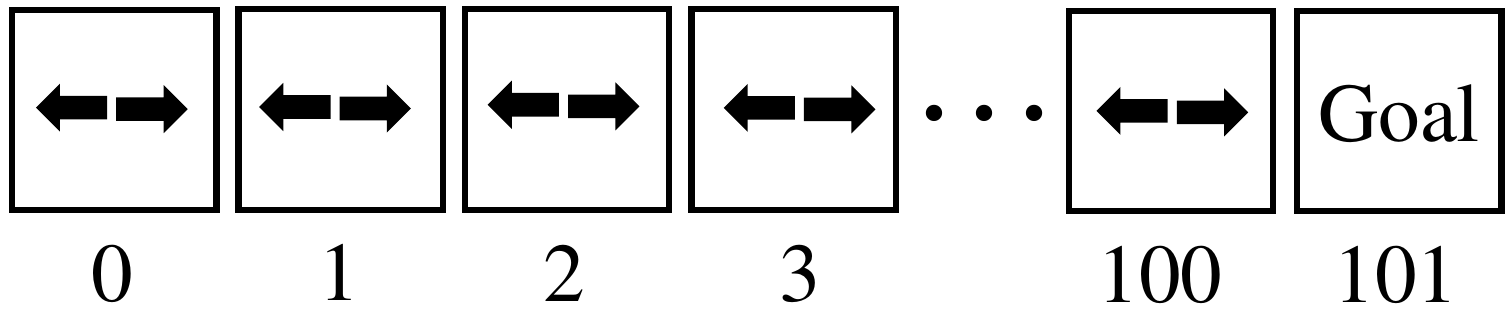}
            \label{fig:policy_bear}
        }
        \subfigure[TD3+BC on lineworld-hard]{
            \includegraphics[width=0.31\linewidth]{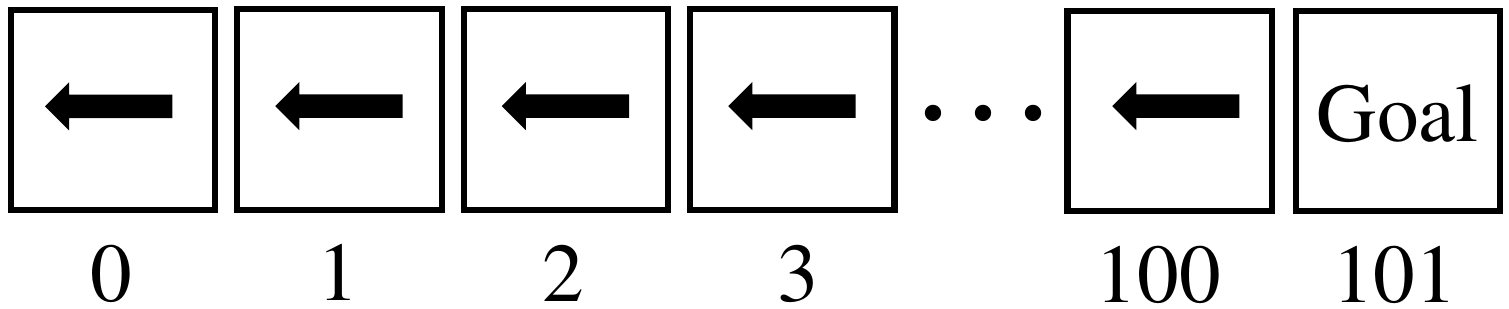}
            \label{fig:policy_td3bc}
        }
        \subfigure[PRDC on lineworld-superhard]{
            \includegraphics[width=0.31\linewidth]{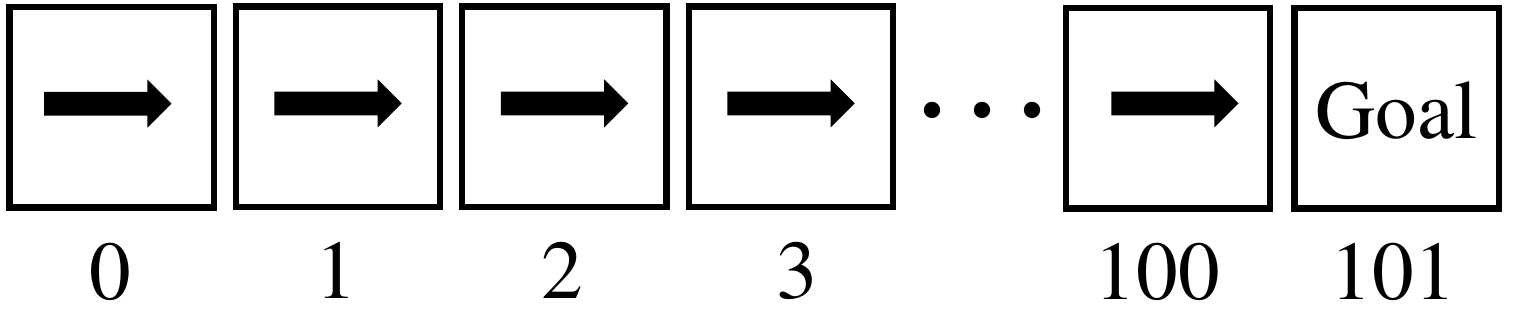}
            \label{fig:policy_prdc}
        }
        \caption{Visualization of the learned policy by BEAR, TD3+BC, and PRDC on selected tasks. (a) On the lineworld-medium task, BEAR outputs $[-\epsilon, \epsilon]$ whatever the input state is, where $\epsilon$ is a small positive number. (b) On the lineworld-hard task, TD3+BC always outputs $-1$. (c) On the lineworld-superhard task, PRDC always outputs $+1$. That is, only PRDC has learned an optimal policy.}
        \label{fig:policy}
    \end{center}
\end{figure*}

First,we choose a set of locomotion and navigation tasks from the D4RL benchmark \cite{d4rl} to be the testbed of performance comparison. All datasets take the "-\textbf{v2}" version. Below is a brief introduction to the baselines:

$\bullet$ \textit{BC} \cite{bc}, which uses mean squared error (MSE) minimization to regress the behavior policy.

$\bullet$ \textit{BCQ} \cite{bcq}, which firstly uses CVAE \cite{cvae} to fit $\mu$ and then learns a $\pi$ similar to $\mu$.
    
$\bullet$ \textit{BEAR} \cite{bear}, which regularizes the policy with its MMD with a Gaussian kernel to the behavior policy.

$\bullet$ \textit{AWAC} \cite{awac}, which applies a KL divergence constraint in the policy improvement step.

$\bullet$ \textit{CQL} \cite{cql}, which learns a conservative Q-function that lower-bounds the policy's true value.

$\bullet$ \textit{IQL} \cite{iql}, which uses expectile regression\footnote{See \url{http://www.sp.unipg.it/surwey/events/28-tutorial.html} for a tutorial on expectile regression.} to learn the policy without evaluating OOD actions.

$\bullet$ \textit{TD3+BC} \cite{td3_bc}, which adds an additional BC regularization term to TD3's policy update loss.

$\bullet$ \textit{DOGE} \cite{doge}, which constrains the policy with the convex hull of the offline dataset via a state-conditioned distance function.

$\bullet$ \textit{SPOT} \cite{spot}, which explicitly models the support set of $\pi$ and presents a density-based regularization.

We then train PRDC for $1$M steps over five seeds on every dataset, with the implementation details deferred to \cref{sec:exp_details}. The average normalized scores in the final ten evaluations are shown in \cref{tab:benchmark}, where we use the scores from \cite{offline_overview} for AWAC, the scores from \cite{spot} for SPOT, and all other baselines take the results reported in \cite{doge}. From \cref{tab:benchmark}, we can see that PRDC performs well and achieves state-of-the-art performance in \textbf{13 out of 18} total tasks.

\subsection{Generalization}\label{sec:generalization}

To test whether PRDC can learn optimal actions even when they do not appear with the current state but only with different states in the dataset, we create a lineworld environment inspired by \cite{toy_env}. It is a deterministic environment, with $102$ cells of length $1$ connected from left to right. The state space is continuous and in $[0, 101]$\footnote{We use $[a,b]$ to denote the set $\{x|a\le x \le b, x\in \mathbb{R}\}$ and $(a,b]$ to denote the set $\{x|a < x \le b, x\in \mathbb{R}\}.$}, where the initial state randomly falls at $[0,1]$. The action space is also continuous and in $[-1,1]$, where a positive action means the agent goes right and a negative action means it goes left. At every step, the reward is 100 if the agent reaches the goal, i.e., its state falls at $(100, 101]$; otherwise the reward is 0. The episode ends only if the agent reaches the goal or the episode length exceeds $105$.

\begin{table}[h]
    \caption{Accomplishments of BEAR, TD3+BC, and PRDC on lineworld-easy, lineworld-medium, lineworld-hard and lineworld-superhard. $\surd$ means accomplish, while $\times$ means not.}
    \label{tab:lineworld}
    \begin{center}
        \begin{tabular}{l|c|c|c} 
            \toprule
            & BEAR & TD3+BC & PRDC \\
            \hline
            lineworld-easy & $\surd$ & $\surd$ & $\surd$ \\
            \hline
            lineworld-medium & $\times$ & $\surd$ & $\surd$ \\
            \hline
            lineworld-hard & $\times$ & $\times$ & $\surd$\\
            \hline
            lineworld-superhard & $\times$ & $\times$ & $\surd$ \\
            \midrule
        \end{tabular}
    \end{center}
\end{table}

We collect four datasets on lineworld. They are named \textit{lineworld-easy}, \textit{lineworld-medium}, \textit{lineworld-hard}, and \textit{lineworld-superhard} in order based on their difficulties. A visual illustration and description of them are in \cref{fig:lineworld}. We then train BEAR, TD3+BC, and PRDC for $10$k steps on each of the four datasets and evaluate the learned policies for ten episodes over five seeds when the training finishes. Moreover, the method accomplishes the task only when the learned policy successfully reaches the goal in all the $5 \times 10$ evaluations. The result is shown in \cref{tab:lineworld}, where we can see that only PRDC can always learn an optimal policy even on the most difficult lineworld-superhard task.

The policies learned by all three methods are visualized in \cref{fig:policy}. BEAR fails on line-medium where the ratio of action $-1$ and $+1$ is $1:1$. BEAR always outputs actions approaching $0$, the average of $-1$ and $+1$. The MMD distance minimization objective in BEAR will converge to action $0$ if not combined with RL optimization. While in this lineworld environment, the reward is sparse, and the regularization is too conservative. Thus the RL optimization may make little difference. The same reason applies to why TD3+BC always outputs actions approaching $-1$ on lineworld-hard, where the ratio of action $-1$ and $+1$ is $99:1$. On the other hand, PRDC can learn an optimal policy that always outputs action $+1$. But from \cref{fig:superhard}, we see that there are no action $+1$ if the state is in $\{0, 2, 4, \cdots, 100\}$ and it only comes with states in $\{1, 3, 5, \cdots, 99\}$. Thanks to the softer nearest neighbor regularization, PRDC can learn optimal actions even when they never occur with the current state.

\subsection{Value Estimation Error}\label{sec:value_error}

\begin{figure}[t]
    \begin{center}
        \includegraphics[width=0.9\columnwidth]{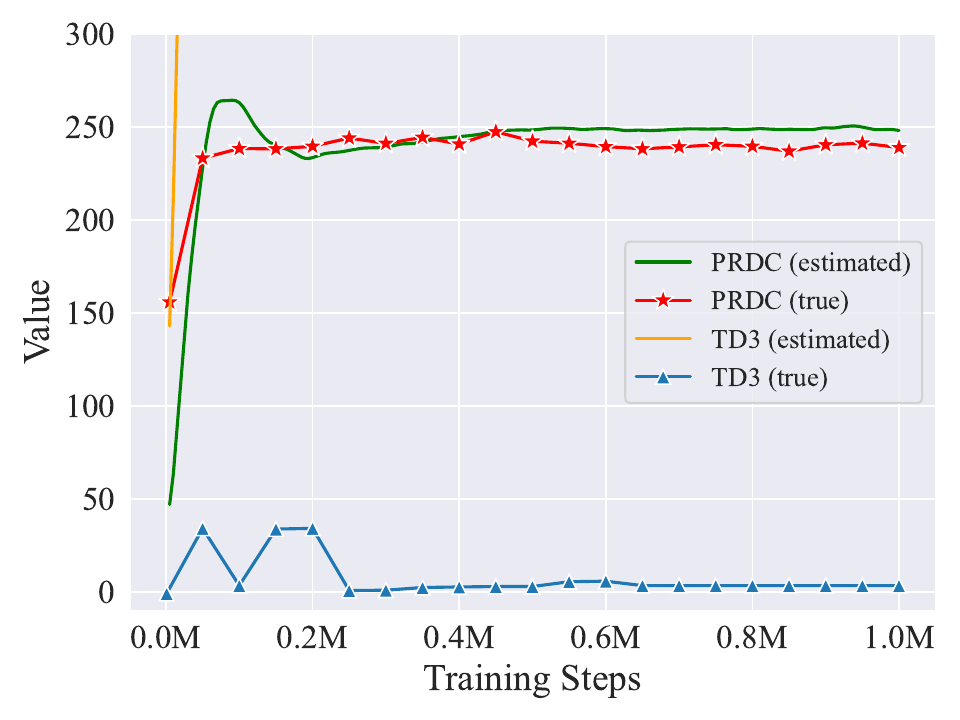}
        \caption{Comparison of PRDC's estimated value, PRDC's true value, TD3's estimated value, and TD3's true value.}
        \label{fig:value_estimation}
    \end{center}
\end{figure}

As discussed in \cref{sec:over_value}, value overestimation is a detrimental issue in offline RL. Theoretical analysis in \cref{sec:theory} has shown that PRDC can alleviate this issue. Here is some empirical evidence. Since PRDC is built upon TD3, we train PRDC and TD3 on the hopper-medium-v2 dataset for $1$M steps. During training, we randomly sample $10$ states from the initial state distribution every $5$k steps and predict actions on these states by the current policy. We then get these state-action pairs' mean estimated Q-value. True values are gotten by Monte Carlo roll out \cite{rlbook}. The result is shown in \cref{fig:value_estimation}, where we see that our proposed point-to-set distance regularization effectively improves the value overestimation problem.

Regarding the point-to-set distance, we also observe an interesting phenomenon. We re-train AWAC, CQL, TD3+BC, and IQL on the hopper-medium-v2 dataset for $3$k steps. All methods take the implementations from CORL~\cite{corl}. During training, we evaluate the policy every $100$ steps. We collect all the generated state-action pairs in the evaluation and calculate their mean point-to-set distance with $\beta=1$. \cref{fig:distance} illustrates the result. From it we can observe an inverse relationship between the policy's normalized return and the mean point-to-set distance, which empirically motivates us that the point-to-set distance minimization objective is a reasonable target. Moreover, we speculate that this inverse relationship exists because a small mean point-to-set distance corresponds to an accurate value function, on which the policy update will return us a safely improved policy.

\begin{figure}[t]
    \begin{center}
        \includegraphics[width=0.9\columnwidth]{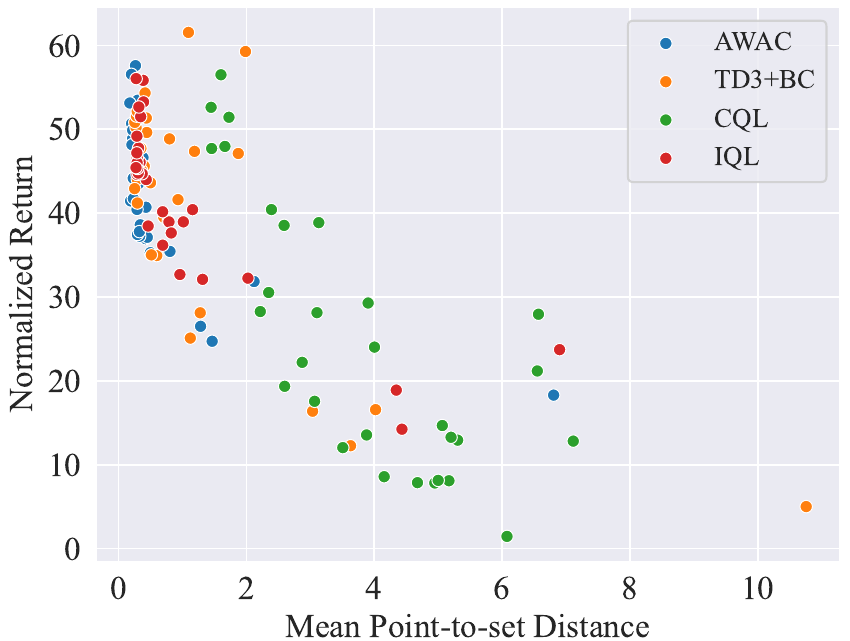}
        \caption{Normalized returns and mean point-to-set distances of baseline methods during training.}
        \label{fig:distance}
    \end{center}
    \vskip -0.2in
\end{figure}

\subsection{Sensitivity on the Hyper-parameter \texorpdfstring{$\beta$}{beta}}\label{sec:beta_ablation}

The hyper-parameter $\beta$ controls how much conservatism is imposed when updating policy. To see how it influences the policy's performance, we conduct an ablation study of $\beta$ on the hopper, halfcheetah, and walker2d datasets. We train PRDC on each dataset for $1$M steps with $\beta \in \{0.1, 1, 2, 5, 10, 1000\}$ and keep other hyper-parameters the same. \cref{fig:main_ablation_beta_hopper} shows some of the results, with the complete results deferred to \cref{sec:more_results}.

\begin{figure}[ht]
    \begin{center}
        \includegraphics[width=0.9\columnwidth]{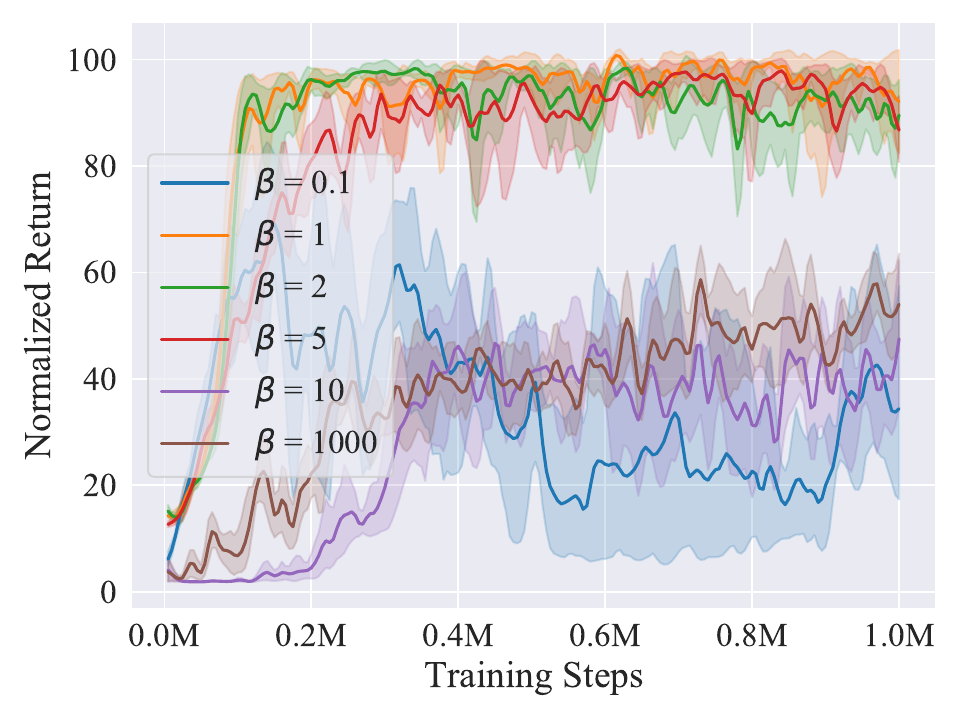}
        \label{fig:main_ablation_beta_hopper}
        \caption{Performances of PRDC on hopper-medium-v2 with different $\beta$, a hyper-parameter in \cref{def:psd}.}
    \end{center}
    \vskip -0.2in
\end{figure}

From \cref{fig:main_ablation_beta_hopper}, we can see: $(\romannumeral1)$ When $\beta$ is $0.1$, PRDC performs bad. It is because a small $\beta$ will ignore the difference in the state when searching for the nearest neighbor. However, in-distribution actions are coupled with states. Thus just constraining the policy to output an action that has occurred in the dataset may not be enough to avoid OOD actions. $(\romannumeral2)$ When $\beta$ is $1000$, PRDC performs not well. A large $\beta$ is too conservative. $(\romannumeral3)$ When $\beta \in \{1,2,5\}$, PRDC performs well. We also observe that PRDC has high performance in a wide range of $\beta$ on all datasets, meaning we do not need much effort to tune it.

\subsection{Run Time}\label{sec:run_time}

Time complexity is also a challenge in offline RL. We run PRDC and baselines on the same dataset and machine for $1$M steps. Baseline implementations are the same as those in \cref{sec:value_error}. The result is shown in \cref{fig:time_cost}, where we can read that PRDC runs at an acceptable speed and even performs faster than CQL on the hopper-medium-v2 dataset. Thanks to KD-Tree, we have a highly efficient implementation of PRDC.

\begin{figure}[ht]
    \begin{center}
        \includegraphics[width=0.9\columnwidth]{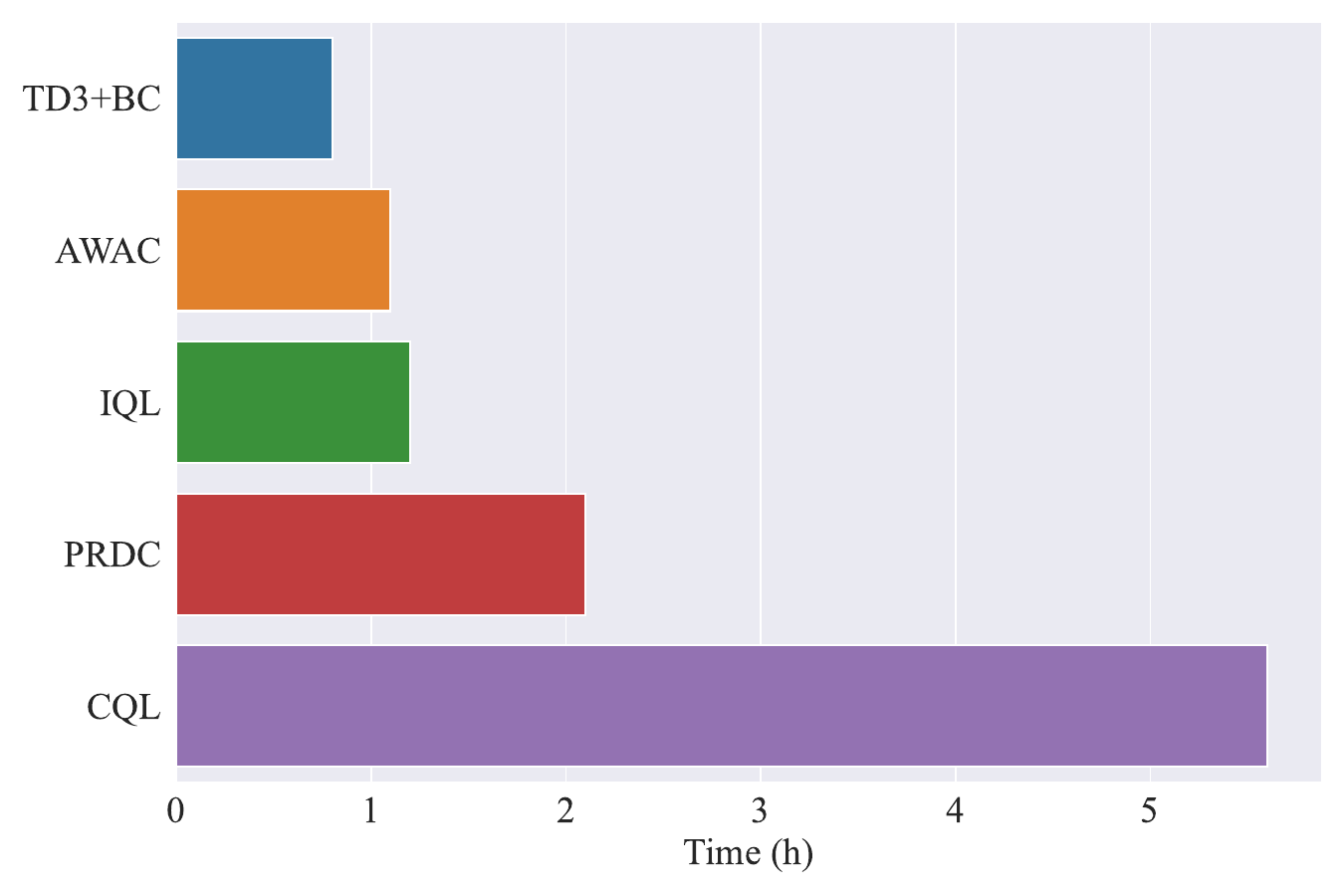}
        \caption{Run time of PRDC and baselines on hopper-medium-v2.}
        \label{fig:time_cost}
    \end{center}
    \vskip -0.2in
\end{figure}

\section{Discussion}\label{sec:discuss}

Currently, the dataset constraint is implemented as the nearest neighbor restriction. However, our fundamental idea of dataset constraint is to leverage as much information as possible while constraining the policy into the dataset. To this end, one direct idea is to search for more than one nearest neighbor and use them together to construct a new constraint. Nevertheless, how to aggregate multiple nearest neighbors needs further consideration. We made a naive attempt by regularizing the policy towards the $k$ nearest neighbors' \textit{average action} with different $k \in \{1, 2, 4\}$. Due to the space limit, we defer the result to \cref{fig:full_ablation_k} in \cref{sec:more_results}. From \cref{fig:full_ablation_k}, we see that this average action restriction does not improve or even hurt the performance, perhaps due to the non-optimal ones in actions of the $k$ nearest neighbors. Indeed, the aggregation of regularization on multiple nearest neighbors is a promising direction worth in-depth study. Moreover, parameterizing the searching process as a neural network and learning to retrieve information optimally from the dataset is also worth investigating.

Although we have shown that an efficient implementation can effectively speed up PRDC in \cref{fig:time_cost}, the nearest neighbor retrieval may still be a bottleneck to minimizing run time, especially when the dataset is large or the states are high-dimensional images. Some ideas from existing studies may help. For example, we can learn a low-dimensional representation, which is then used to calculate the point-to-set distance. We can also use methods like SCaNN for fast approximate nearest neighbor retrieval. 

These further discussions are beyond the scope of this paper, and we leave them for future works.

\section{Conclusion}

We propose dataset constraint, a new type of policy regularization method. Unlike the commonly used distribution constraint and support constraint, which limit the policy to actions that have occurred with the current state, the proposed dataset constraint is less conservative and allows the policy to learn from all actions in the dataset. Empirical and theoretical evidence shows that the dataset constraint can effectively alleviate offline RL's fundamentally challenging value overestimation issue. It is simple and effective with new state-of-the-art performance on the D4RL \cite{d4rl} benchmark but only one additional hyper-parameter compared to \cite{td3_bc}. Overall, this is the first attempt to constrain the policy from a dataset perspective in offline RL. We hope our work can inspire more relevant research as stated in \cref{sec:discuss}.

\section*{Acknowledgements}
This work is supported by the National Key R\&D Program of China (2022ZD0114804), the National Science Foundation of China (62276126), and the Fundamental Research Funds for the Central Universities (14380010). The authors would like to thank Chengxing Jia, Chenxiao Gao, Chenghe Wang, Feng Chen and the anonymous reviewers for their support and helpful discussions on improving the paper. 

\bibliography{references}
\bibliographystyle{icml2023}

\newpage
\appendix
\onecolumn
\section{Proofs}\label{sec:proofs}

\begin{lemma}\label{lem:pos_ineq}
For $\forall a, b \in\mathbb{R}_+$, the following inequality holds,
$$
\sqrt{a^2 + b^2} \le a + b.
$$
\end{lemma}

\begin{lemma}[Triangle inequality]\label{lem:tri_ineq}
For $\forall x, y \in \mathbb{R}^m$, the following inequality holds,
$$
\|x + y\| \le \|x\| + \|y\|.
$$
\end{lemma}

\begin{lemma}\label{lem:error}
For $\forall (s,a), (\hat{s},\hat{a})\in\mathcal{S}\times\mathcal{A}$, if $\|(\beta s)\oplus a-(\beta \hat{s})\oplus \hat{a}\|\le \epsilon, \beta > 0, \epsilon > 0$, then the following inequalities hold,
\begin{itemize}
    \item $\|s-\hat{s}\|\le \epsilon/\beta$,
    \item $\|a-\hat{a}\|\le\epsilon$,
    \item $\|s\oplus a- \hat{s}\oplus \hat{a}\|\le (1/\beta + 1)\epsilon$.
\end{itemize}
\end{lemma}

\begin{proof}
Since $\beta>0$, we have
$$
\beta \|s-\hat{s}\| \le \|(\beta s)\oplus a-(\beta \hat{s})\oplus \hat{a}\| \le \epsilon.
$$
By dividing both sides of the above inequality by $\beta$, we achieve $\|s-\hat{s}\|\le \epsilon/\beta$. Similarly, we can get $\|a-\hat{a}\|\le\epsilon$. Thus,
$$
\|s\oplus a- \hat{s}\oplus \hat{a}\|\le \|s-\hat{s}\| + \|a-\hat{a}\| \le (1/\beta + 1)\epsilon,
$$
where the first inequality comes from \cref{lem:pos_ineq}.
\end{proof}

\thmq*

\begin{proof}
Let
$$
(\hat{s}, \hat{a}) = \argmin_{(\hat{s},\hat{a}) \in \mathcal{D}} d_\mathcal{D}^\beta (s, \pi(s)),
$$
where $\hat{a} = \mu(\hat{s})$. By expanding the left side of \cref{equ:q_smooth}, we get
$$
\begin{aligned}
\|Q(s,\pi_\phi(s)) - Q(s,\mu(s))\|
&= \|Q(s,\pi_\phi(s)) - Q(\hat{s}, \hat{a})+Q(\hat{s},\hat{a})-Q(s,\mu(s))\|\\
&\overset{(\romannumeral1)}{\le} \|Q(s,\pi_\phi(s)) - Q(\hat{s}, \hat{a})\| + \|Q(\hat{s},\hat{a})-Q(s,\mu(s))\|\\
&\overset{(\romannumeral2)}{\le} K_Q(\|s\oplus \pi(s) - \hat{s} \oplus \hat{a}\| + \|\hat{s}\oplus\mu(\hat{s})-s\oplus \mu(s)\|) \\
&\overset{(\romannumeral3)}{\le} K_Q \left((1/\beta+1)\epsilon + \|\hat{s}\oplus\mu(\hat{s}) - s \oplus \mu(s)\|\right)\\
&\overset{(\romannumeral4)}{\le} K_Q\left((1/\beta+1)\epsilon + \|\hat{s} - s\| + \|\mu(\hat{s}) - \mu(s)\| \right) \\
&\overset{(\romannumeral5)}{\le} K_Q \left((1/\beta+1)\epsilon + \|\hat{s} - s\|+K_\mu \|\hat{s}-s\|\right)\\
&\overset{(\romannumeral6)}{\le} \left((K_\mu+2)/\beta + 1\right)K_Q\epsilon.
\end{aligned}
$$
Here, $(\romannumeral1)$ is due to \cref{lem:tri_ineq}; $(\romannumeral2)$ is due to \cref{assum:q_continue}; $(\romannumeral3)$ is due to \cref{lem:error}; $(\romannumeral4)$ is due to \cref{lem:pos_ineq}; $(\romannumeral5)$ is due to \cref{assum:mu_continue}; and $(\romannumeral6)$ is due to \cref{lem:error}.
\end{proof}

\begin{lemma}\label{lem:abs_inq}
For function $f: S\subset \mathbb{R}^m \to \mathbb{R}$, we have that
$$
\abs{\int_a^b f(x) \dd x} \le \int_a^b \abs{f(x)}\dd x,
$$
where $a\le b$, and $[a,b] \subseteq S$.
\end{lemma}
\begin{proof}
For any $ x \in S$, the following inequality holds,
$$
-\abs{f(x)} \le f(x) \le \abs{f(x)}.
$$
Thus,
$$
-\int_a^b \abs{f(x)} \dd x \le \int_a^b f(x) \dd x \le \int_a^b \abs{f(x)} \dd x.
$$
Therefore,
$$
\abs{\int_a^b f(x) \dd x} \le \int_a^b \abs{f(x)}\dd x.
$$
Thus, we finish the proof.
\end{proof}

\begin{lemma}[Lemma 1 of \cite{ddpg_convergence}]\label{lem:d_bound}
With \cref{assum:p_continue}, the following inequality holds,
$$
\int_\mathcal{S} \abs{d^\pi(s) - d^\mu(s)}\dd s \le CK_\mathcal{P} \max_{s\in \mathcal{S}}\|\pi(s)-\mu(s)\|,
$$
where $C$ is a positive constant.
\end{lemma}
\begin{proof}
Please see the appendix of \cite{ddpg_convergence}.
\end{proof}

\begin{lemma}\label{lem:pi_diff}
Let $\max_{s\in\mathcal{S}} d^\beta_\mathcal{D}(s,\pi_\phi(s))\le \epsilon$, which can be acvieved by PRDC. Then with \cref{assum:mu_continue}, we have
$$
\|\pi(s) - \mu(s)\| \le (1+\frac{1}{K_\mu}{\beta})\epsilon, \forall s \in \mathcal{S}, 
$$
\end{lemma}
\begin{proof}
We follow a similar proof of \cref{thm:qsmooth}. Let
$$
(\hat{s}, \hat{a}) = \argmin_{(\hat{s},\hat{a}) \in \mathcal{D}} d_\mathcal{D}^\beta (s, \pi(s)), 
$$
where $\hat{a} = \mu(\hat{s})$. Thus
$$
\|\pi(s) - \mu(\hat{s})\|\leq \epsilon.
$$
Then,
$$
\begin{aligned}
\|\pi(s) - \mu(s)\|
&= \|\pi(s) - \mu(\hat{s}) + \mu(\hat{s}) - \mu(s)\|\\
&\overset{(\romannumeral1)}{\le} \|\pi(s) - \mu(\hat{s})\| + \|\mu(\hat{s}) - \mu(s)\| \\
&\overset{(\romannumeral2)}{\le} \|\pi(s) - \mu(\hat{s})\| + K_\mu \|\hat{s} - s\|\\
&\overset{(\romannumeral3)}{\le} (1+\frac{K_\mu}{\beta})\epsilon.
\end{aligned}
$$
Here, $(\romannumeral1)$ is due to \cref{lem:tri_ineq}; $(\romannumeral2)$ is due to \cref{assum:mu_continue}; and $(\romannumeral3)$ is due to \cref{lem:error}.
\end{proof}

\perf*

\begin{proof}
With \cref{lem:tri_ineq}, we have
\begin{equation}
    \begin{aligned}
        \abs{J(\pi^*) - J(\pi)}
    &= \abs{J(\pi^*) - J(\mu) + J(\mu) - J(\pi)}\\
    &\leq \abs{J(\pi^*) - J(\mu)} + \abs{J(\pi) - J(\mu)}.
    \end{aligned}
\end{equation}
We first bound $\abs{J(\pi) - J(\mu)}$. By taking into \cref{equ:dperf}, we get
\begin{equation}\label{equ:thm2_1}
\begin{aligned}
\abs{J(\pi) - J(\mu)}
&= \abs{\frac{1}{1-\gamma}\mathbb{E}_{s\sim d^\pi(s)}\bqty{r(s)} - \frac{1}{1-\gamma}\mathbb{E}_{s\sim d^\mu(s)}\bqty{r(s)}} \dd s \\
&= \frac{1}{1-\gamma}\abs{\int_\mathcal{S}(d^\pi(s) - d^\mu(s))r(s)\dd s}\\
&\overset{(\romannumeral1)}{\le} \frac{1}{1-\gamma}\int_\mathcal{S}\abs{d^\pi(s) - d^\mu(s)}\abs{r(s)}\dd s \\
&\le \frac{R_{\max}}{1-\gamma} \int_\mathcal{S}\abs{d^\pi(s) - d^\mu(s)}\dd s. \\
&\overset{(\romannumeral2)}{\le} \frac{CK_{\mathcal{P}}R_{\max}}{1-\gamma}\max_{s\in \mathcal{S}}\|\pi(s)-\mu(s)\|\\
&\overset{(\romannumeral3)}{\le} \frac{CK_{\mathcal{P}}R_{\max}}{1-\gamma}(1+\frac{K_\mu}{\beta})\epsilon_\pi.
\end{aligned}
\end{equation}
Here, $(\romannumeral1)$ is due to \cref{lem:abs_inq}; $(\romannumeral2)$ is due to \cref{lem:d_bound}; and $(\romannumeral3)$ is due to \cref{lem:pi_diff}.

Similarly, we have that 
\begin{equation}
    \abs{J(\pi^*) - J(\mu)} \leq \frac{CK_{\mathcal{P}}R_{\max}}{1-\gamma} \epsilon_{\rm opt}.
\end{equation}
Thus,
\begin{equation}
    \abs{J(\pi^*)-J(\pi)} \le \frac{CK_{\mathcal{P}}R_{\max}}{1-\gamma}\bqty{(1+\frac{K_\mu}{\beta})\epsilon_\pi + \epsilon_{\rm opt}}.
\end{equation}
The proof is finished.
\end{proof}

\section{Implementation Details}\label{sec:exp_details}

We implement PRDC based on the author-provided implementation of TD3+BC\footnote{\url{https://github.com/sfujim/TD3_BC}}, and the KD-tree implementation comes from SciPy \cite{scipy}. The full hyper-parameters setting is in \cref{tab::para}. We also utilize linear reward transformation (the simplest form of reward shaping~\cite{reward_shift}), a common trick used by CQL, IQL, FisherBRC, etc.~\cite{cql, fisherbrc, iql}, to accelerate training. Specifically, we transform the real reward function $r$ to $\tilde{r}$ by:
$$
    \tilde{r} = \texttt{scale} * r + \texttt{shift},
$$
where $\texttt{scale}$ and $\texttt{shift}$ are hyper-parameters.

\subsection{Software}
We use the following software versions: 
\begin{itemize}
    \item Python 3.8
    \item MuJoCo 2.2.0 \cite{mujoco}
    \item Gym 0.21.0 \cite{gym}
    \item MuJoCo-py 2.1.2.14
    \item PyTorch 1.12.1 \cite{pytorch}
\end{itemize}

\subsection{Hardware}
We use the following hardware:
\begin{itemize}
    \item NVIDIA RTX A4000
    \item 12th Gen Intel(R) Core(TM) i9-12900K
\end{itemize}

\section{More Results}\label{sec:more_results}

\subsection{Sensitivity on the Hyper-paramter $\beta$}

This section investigates how $\beta$ influences the policy's performance. On the hopper, halfcheetah, and walker2d (-random-v2, -medium-replay-v2, -medium-v2, -medium-expert-v2) datasets, we train PRDC for $1$M steps over $5$ seeds with $\beta\in \{0.1, 1, 2, 5, 10, 1000\}$ and keep other hyper-parameters the same. The result is shown in \cref{fig:full_ablation_beta}. Apart from the discussion in \cref{sec:beta_ablation}, we notice that PRDC performs better on the random datasets when $\beta \in \{0.1, 1, 2\}$ than when $\beta \in \{5, 10, 1000\}$. The reason is that a bigger $\beta$ makes the regularization in \cref{equ:reg} more like behavior cloning \cite{bc}. However, on the random datasets, behavior cloning will return a poor policy. On the other hand, a smaller $\beta$ will leave more room for RL optimization.

\subsection{$k$ Nearest Neighbor Constraint}

This section investigates whether more than one nearest-neighbor constraint can improve the policy's performance. To this end, we modify the regularization defined in \cref{equ:reg} to minimize the distance between $\pi(s)$ and the average action of $(s,\pi(s))$'s $k$ nearest neighbors. That is, we define a new regularization as
\begin{equation}
    \min_\phi \mathcal{L}(\phi) := \|\pi(s) - \overline{a}\|,
\end{equation}
where $\overline{a}$ denotes the average action of $(s,\pi(s))$'s $k$ nearest neighbors.

We choose $k \in \{1, 2, 4\}$ and train PRDC for $1$M steps on the hopper, halfcheetah, and walker2d (-random-v2, -medium-replay-v2, -medium-v2, -medium-expert-v2) datasets. The result is shown in \cref{fig:full_ablation_k}, where we see that this average action regularization does not improve the performance on almost the tasks. The learning curves become much more unstable, where we speculate that it is because there are non-optimal ones in the actions of the $k$ nearest neighbors.

\begin{table*}[ht]
    \caption{Full hyper-parameters setting of PRDC.}
    \label{tab::para}
    \begin{center}
        \begin{tabular}{c|c|c} 
            \toprule
            & Hyper-parameters & Value\\
            \midrule
            \multirow{8}{*}{Network} &Actor learning rate & $3 \times 10^{-4}$ \\
            & Critic learning rate & \thead[c]{$3 \times 10^{-4}$ for MuJoCo \\ $1 \times 10^{-3}$ for AntMaze}\\
            & Batch size & 256 \\
            & Optimizer & Adam \\
            & Q-network & 3 layers ReLU activated MLPs with 256 units\\
            & Policy Network & 3 layers ReLU activated MLPs with 256 units\\
            \hline 
            \multirow{8} {*} {TD3} & Critic learning rate & \thead[c]{0.99 for MuJoCo \\ 0.995 for AntMaze}\\
            & Number of iterations & $10^6$ \\
            & Target update rate $\tau$ & 0.005\\
            & Policy noise & 0.2\\
            & Policy noise clipping & 0.5 \\
            & Policy update frequency & 2 \\
            \hline
            \multirow{9} {*} {PRDC} &Normalized state & True \\ 
            & $k$ & 1 \\
            & $\beta$ & \thead[c]{2.0 for MuJoCo \\ \{2.0, 7.5, 15.0\} for AntMaze \{umaze, medium, large\}} \\
            & $\alpha$ & \thead[c]{2.5 for Hopper and Walker2d, 40.0 for HalfCheetah \\ \{2.5, 7.5, 20.0\} for AntMaze \{umaze, medium, large\}} \\
            & linear reward transformation & \thead[c]{$\texttt{scale}=1, \texttt{shift}=0$ for MuJoCo \\ $\texttt{scale}=10000, \texttt{shift}=-1$ for AntMaze } \\
            \bottomrule
        \end{tabular}
    \end{center}
\end{table*}

\begin{figure*}[ht]
    \begin{center}
        \subfigure[walker2d-random-v2]{
            \includegraphics[width=0.3\linewidth]{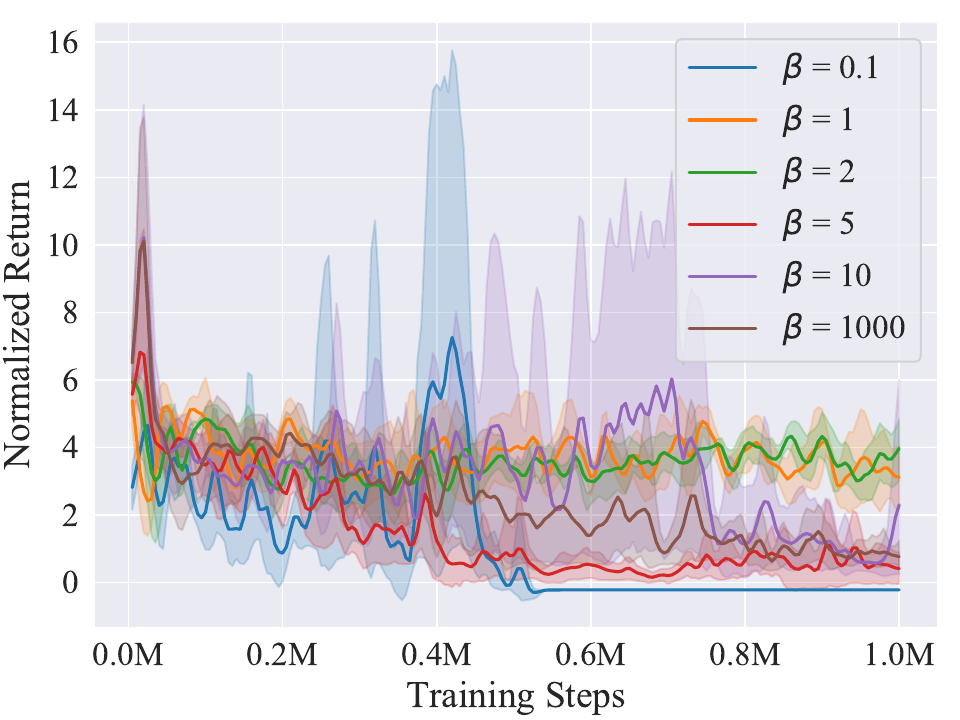}
        }
        \subfigure[halfcheetah-random-v2]{
            \includegraphics[width=0.3\linewidth]{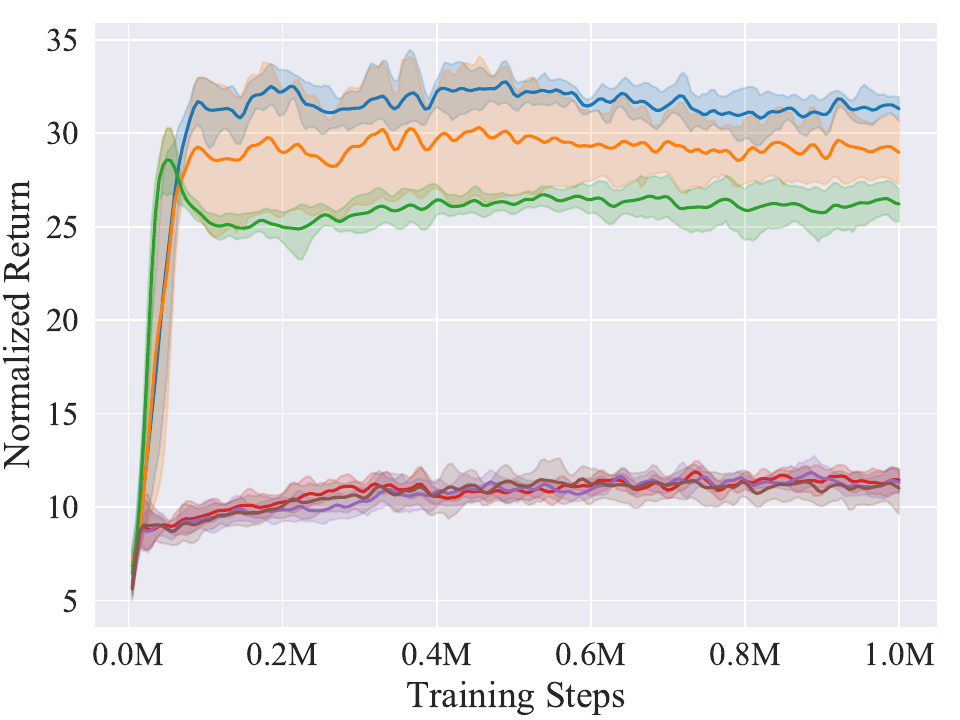}
        }
        \subfigure[hopper-random-v2]{
            \includegraphics[width=0.3\linewidth]{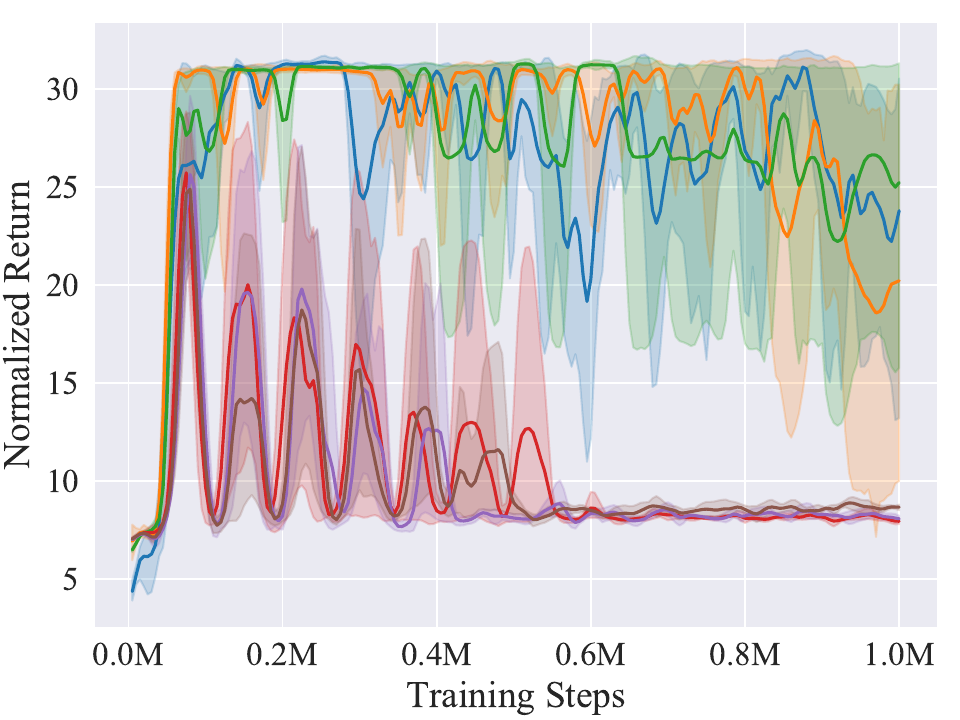}
        }\\
        \subfigure[walker2d-medium-replay-v2]{
            \includegraphics[width=0.3\linewidth]{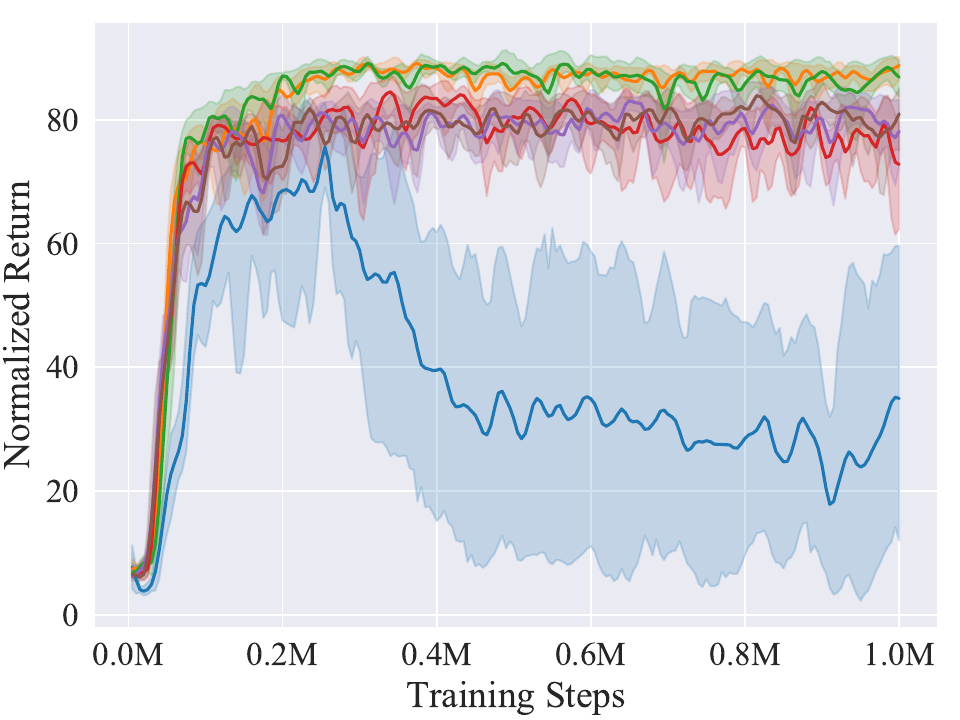}
        }
        \subfigure[halfcheetah-medium-replay-v2]{
            \includegraphics[width=0.3\linewidth]{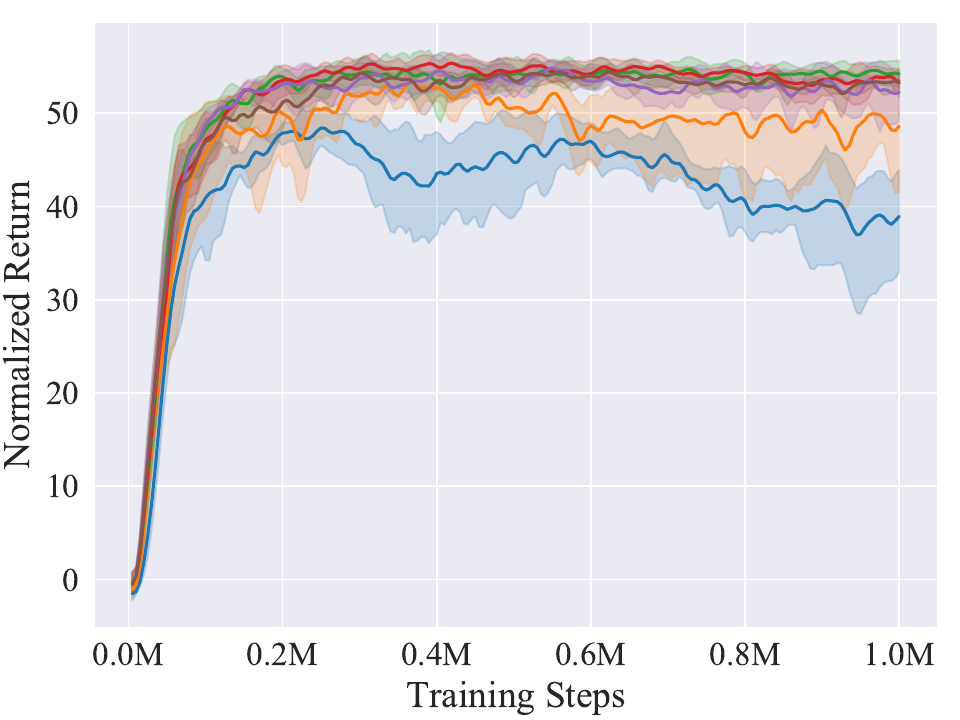}
        }
        \subfigure[hopper-medium-replay-v2]{
            \includegraphics[width=0.3\linewidth]{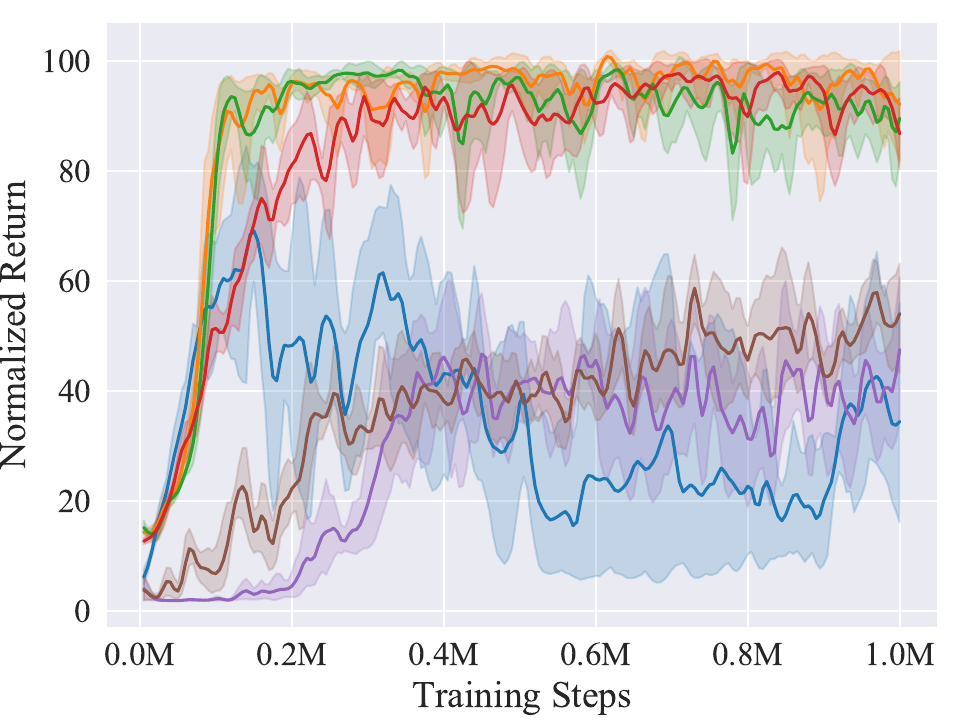}
        }\\
        \subfigure[walker2d-medium-v2]{
            \includegraphics[width=0.3\linewidth]{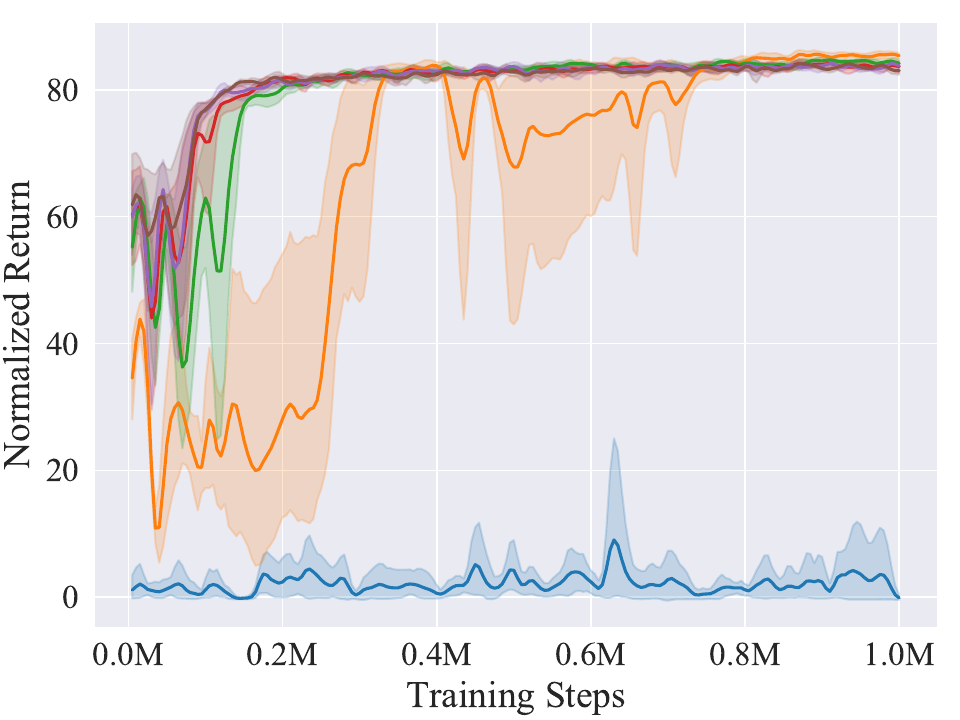}
        }
        \subfigure[halfcheetah-medium-v2]{
            \includegraphics[width=0.3\linewidth]{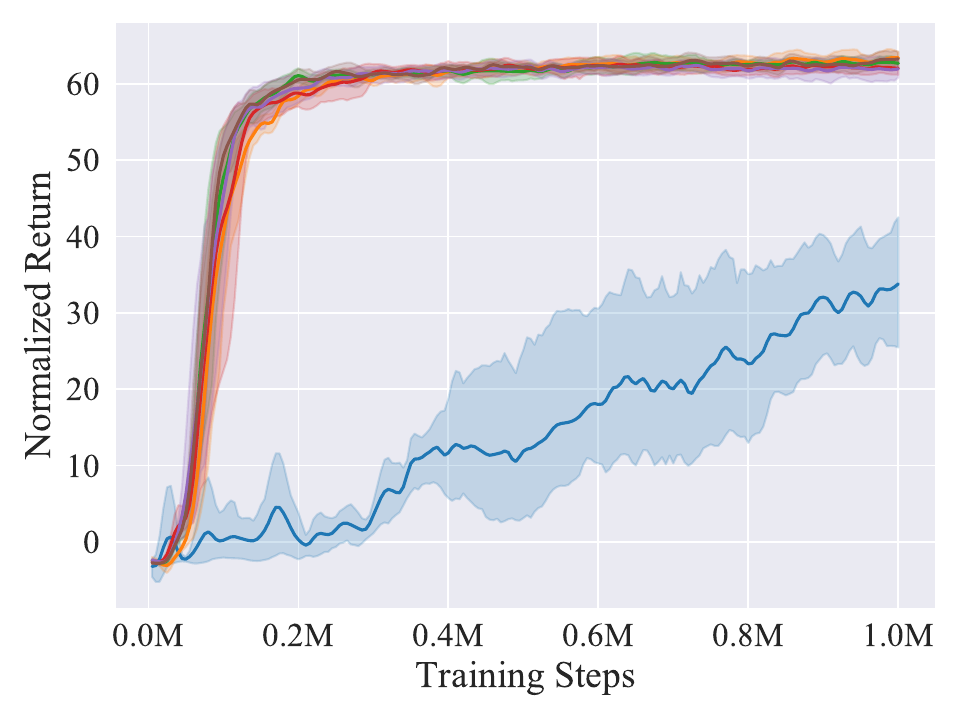}
        }
        \subfigure[hopper-medium-v2]{
            \includegraphics[width=0.3\linewidth]{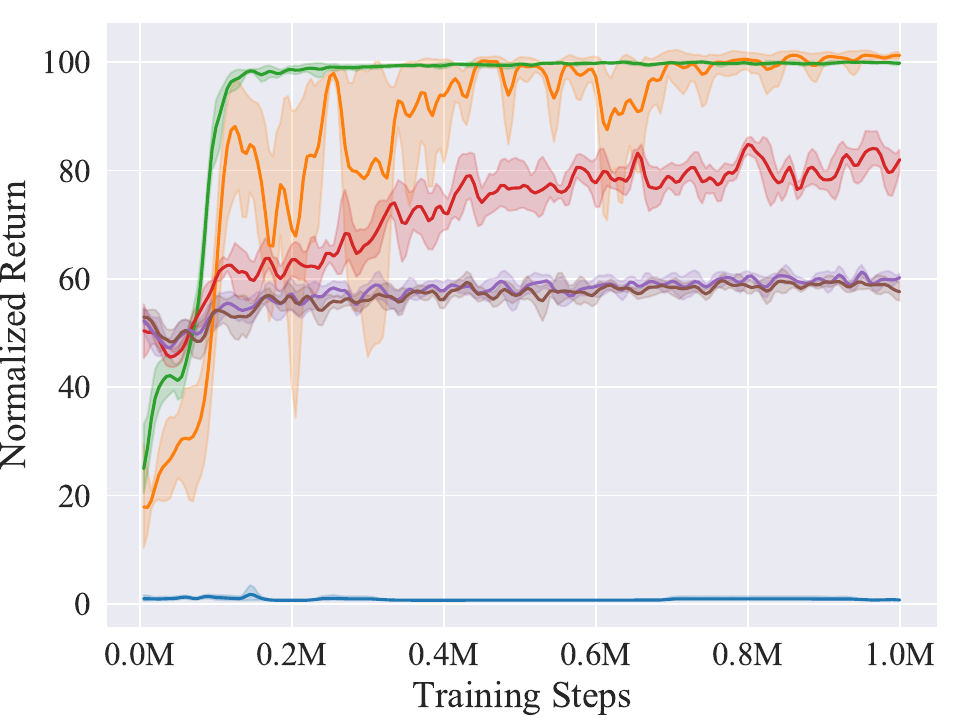}
        }\\
        \subfigure[walker2d-medium-expert-v2]{
            \includegraphics[width=0.3\linewidth]{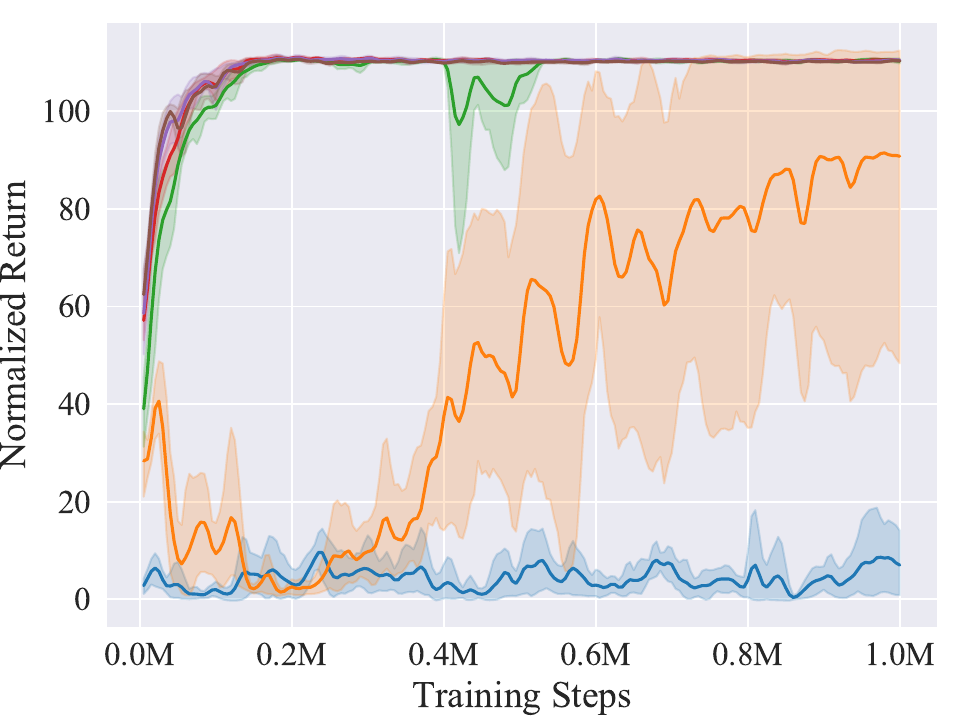}
        }
        \subfigure[halfcheetah-medium-expert-v2]{
            \includegraphics[width=0.3\linewidth]{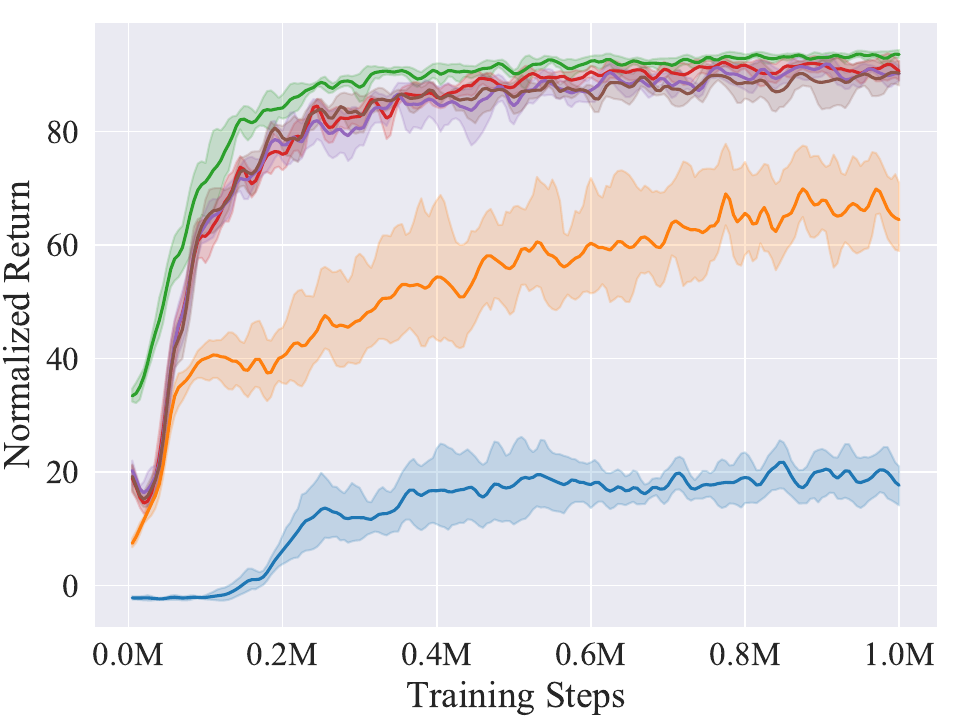}
        }
        \subfigure[hopper-medium-expert-v2]{
            \includegraphics[width=0.3\linewidth]{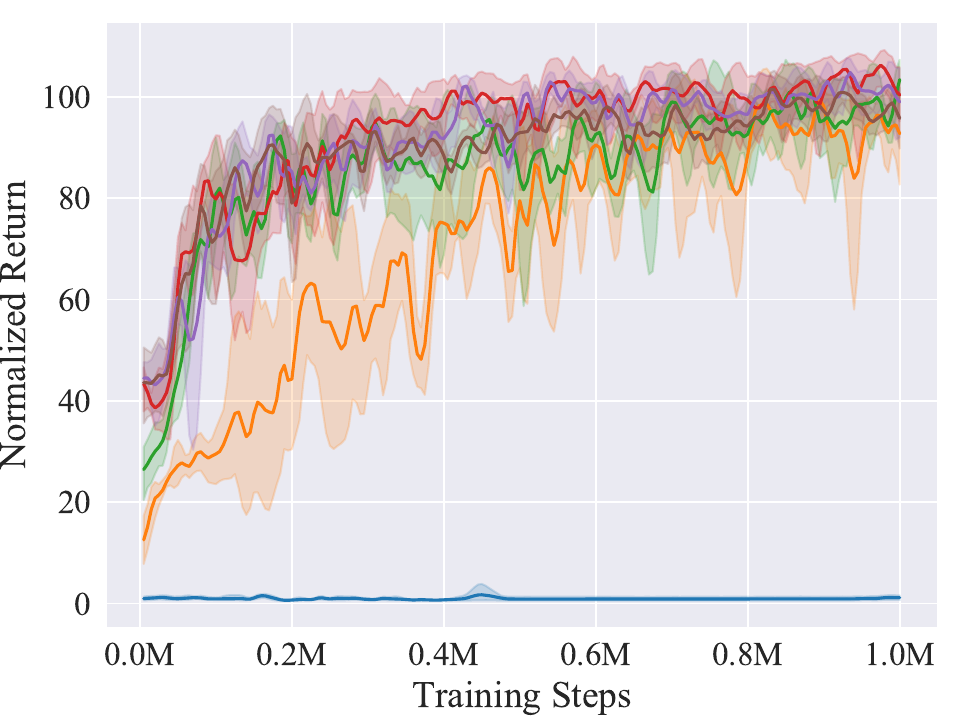}
        }
    \end{center}
    \caption{Performances of PRDC with different $\beta$, a hyper-parameter in \cref{def:psd}.}
    \label{fig:full_ablation_beta}
\end{figure*}

\begin{figure*}[ht]
    \begin{center}
        \subfigure[walker2d-random-v2]{
            \includegraphics[width=0.3\linewidth]{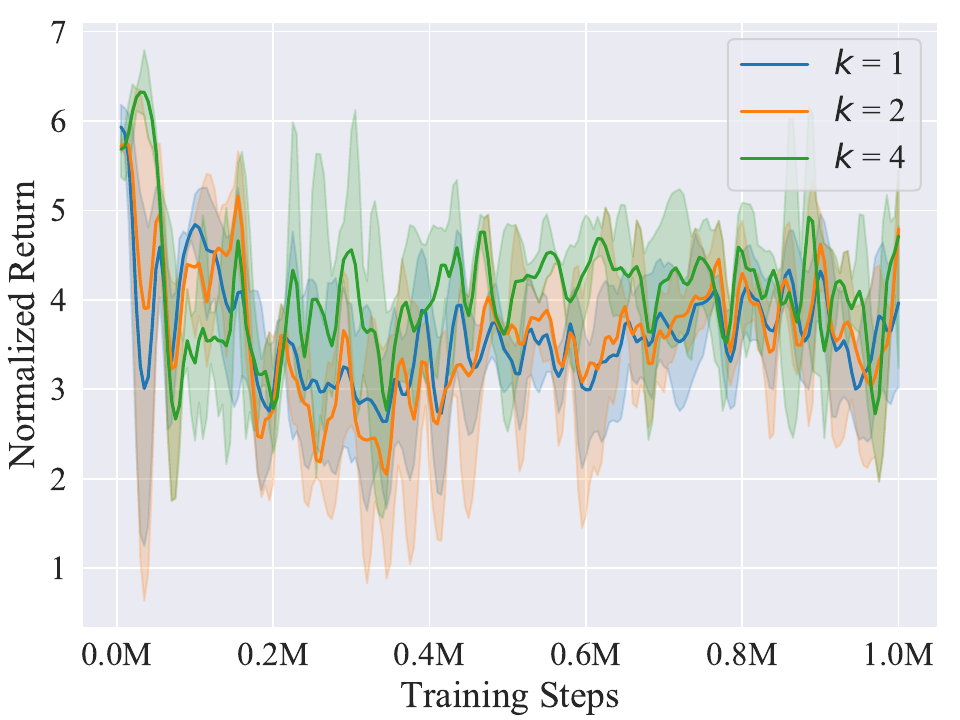}
        }
        \subfigure[halfcheetah-random-v2]{
            \includegraphics[width=0.3\linewidth]{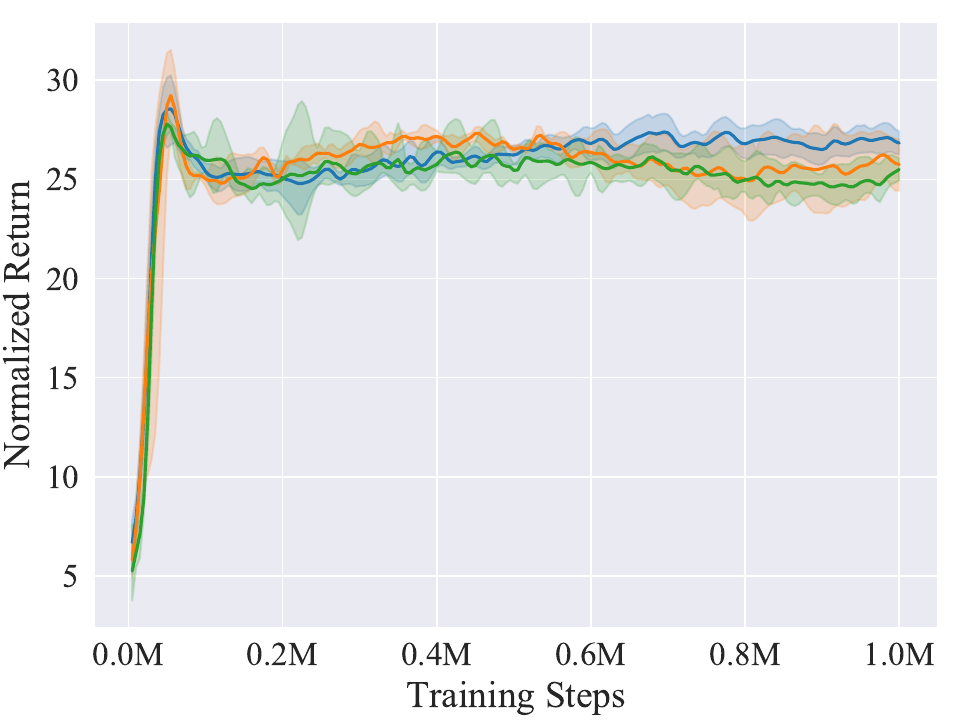}
        }
        \subfigure[hopper-random-v2]{
            \includegraphics[width=0.3\linewidth]{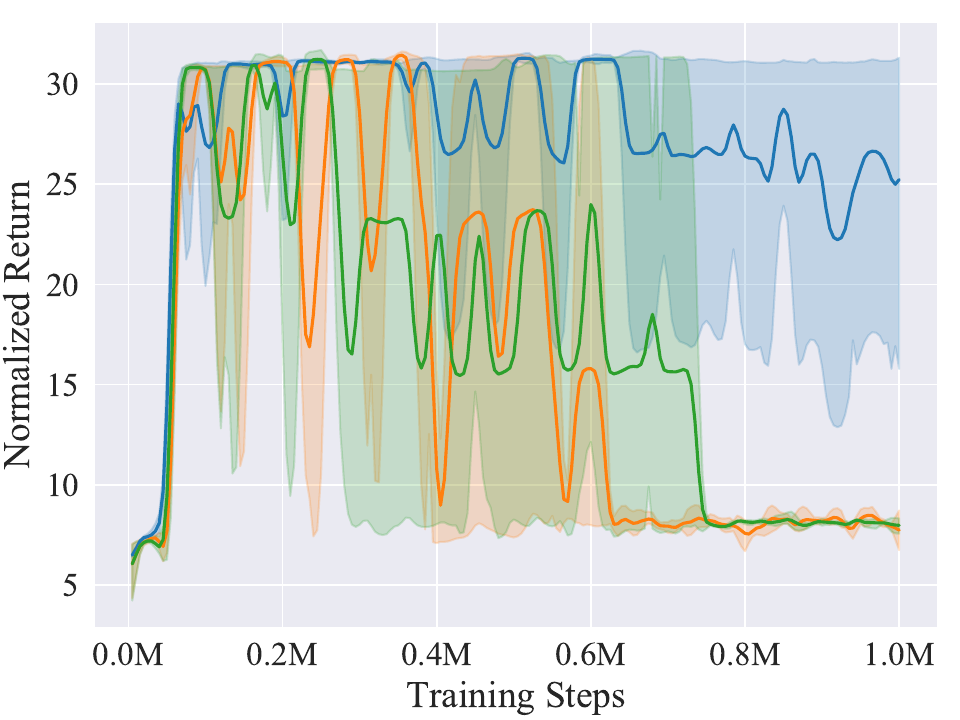}
        }\\
        \subfigure[walker2d-medium-replay-v2]{
            \includegraphics[width=0.3\linewidth]{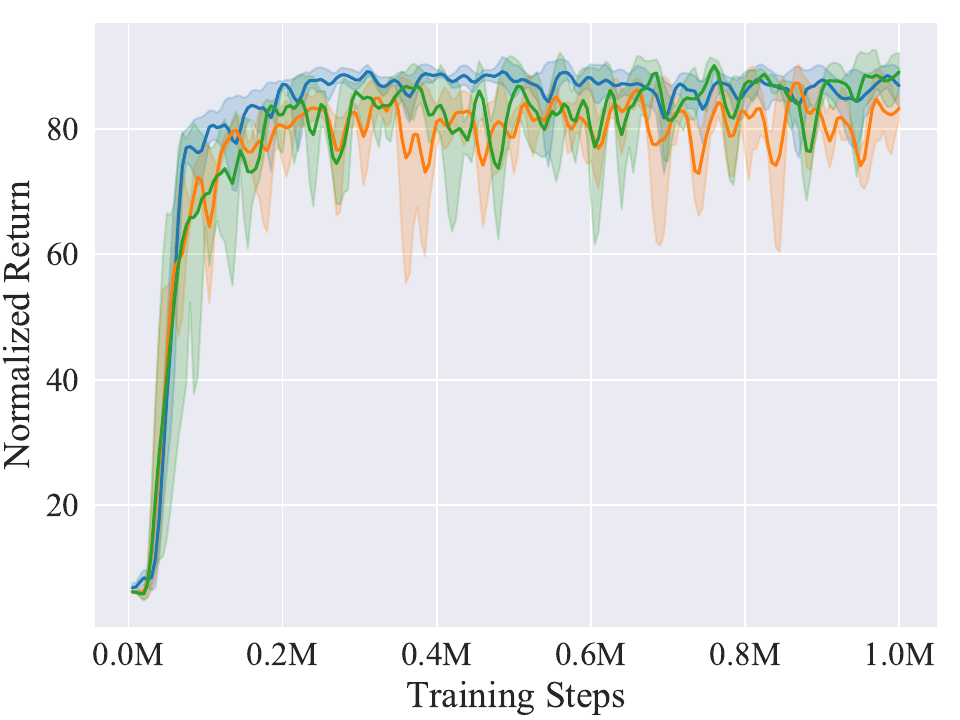}
        }
        \subfigure[halfcheetah-medium-replay-v2]{
            \includegraphics[width=0.3\linewidth]{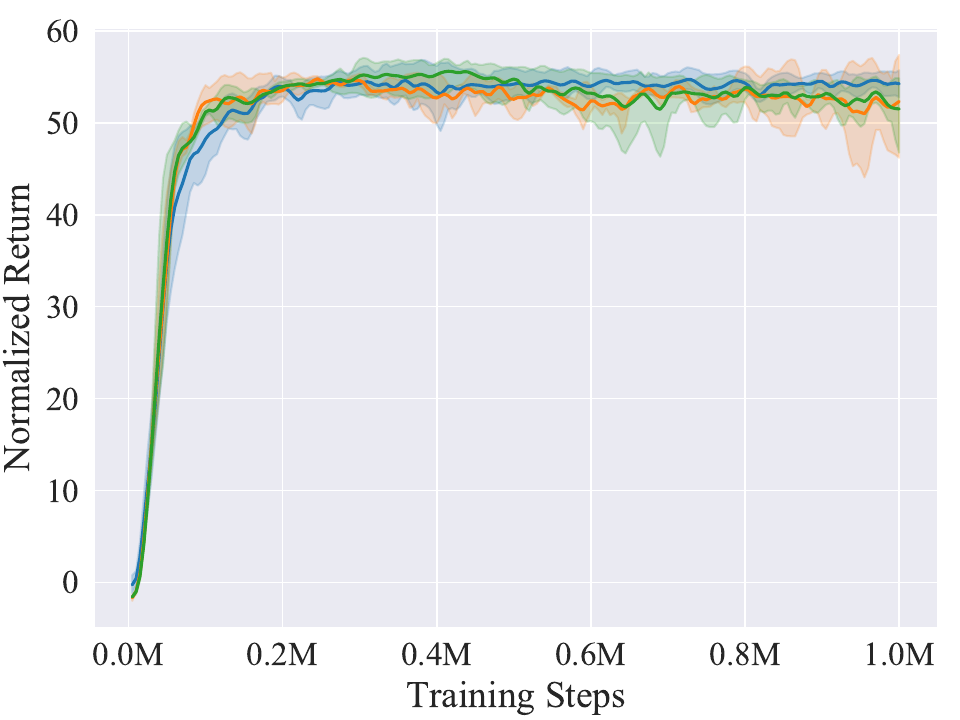}
        }
        \subfigure[hopper-medium-replay-v2]{
            \includegraphics[width=0.3\linewidth]{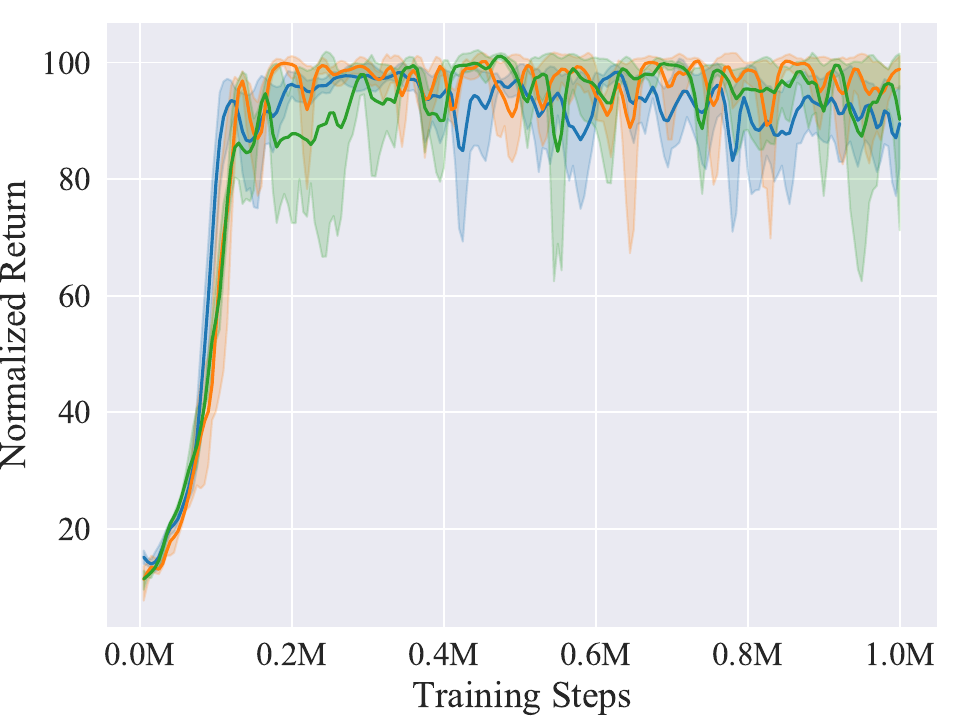}
        }\\
        \subfigure[walker2d-medium-v2]{
            \includegraphics[width=0.3\linewidth]{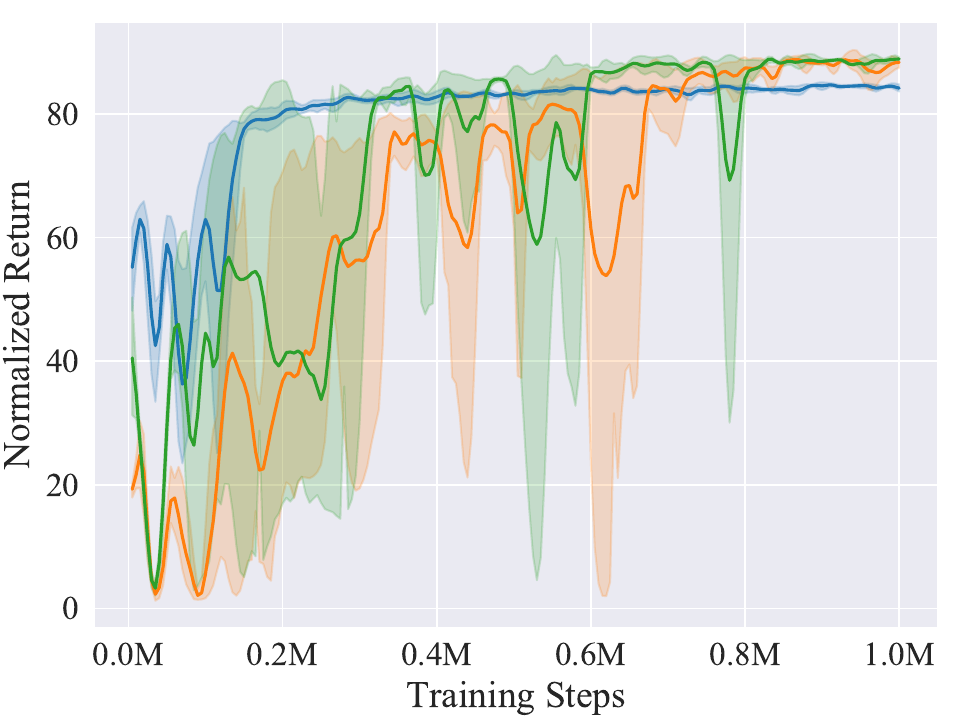}
        }
        \subfigure[halfcheetah-medium-v2]{
            \includegraphics[width=0.3\linewidth]{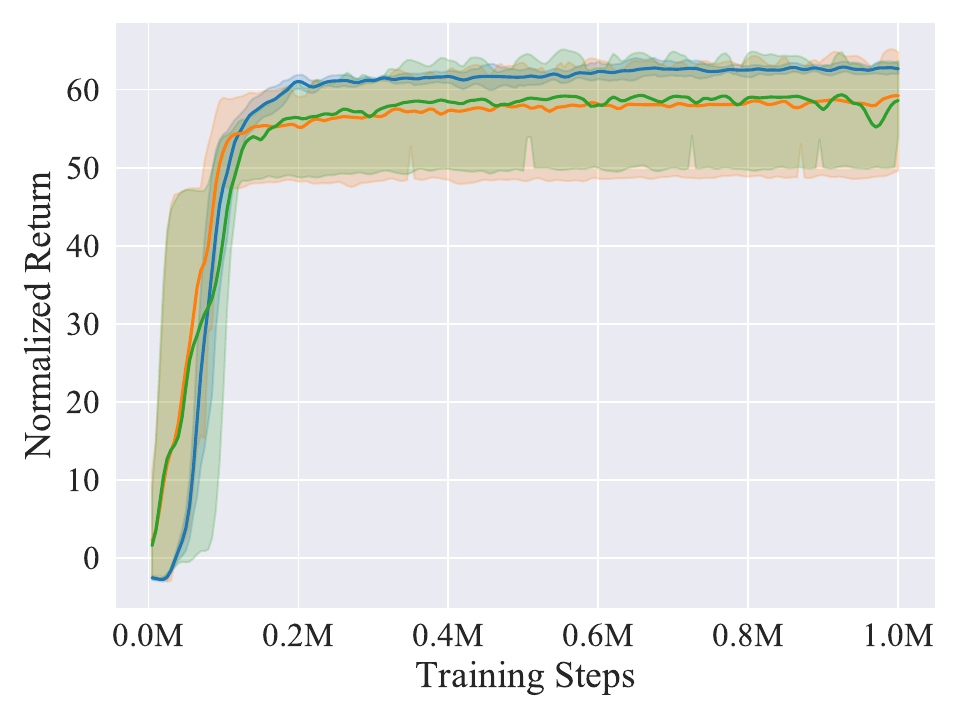}
        }
        \subfigure[hopper-medium-v2]{
            \includegraphics[width=0.3\linewidth]{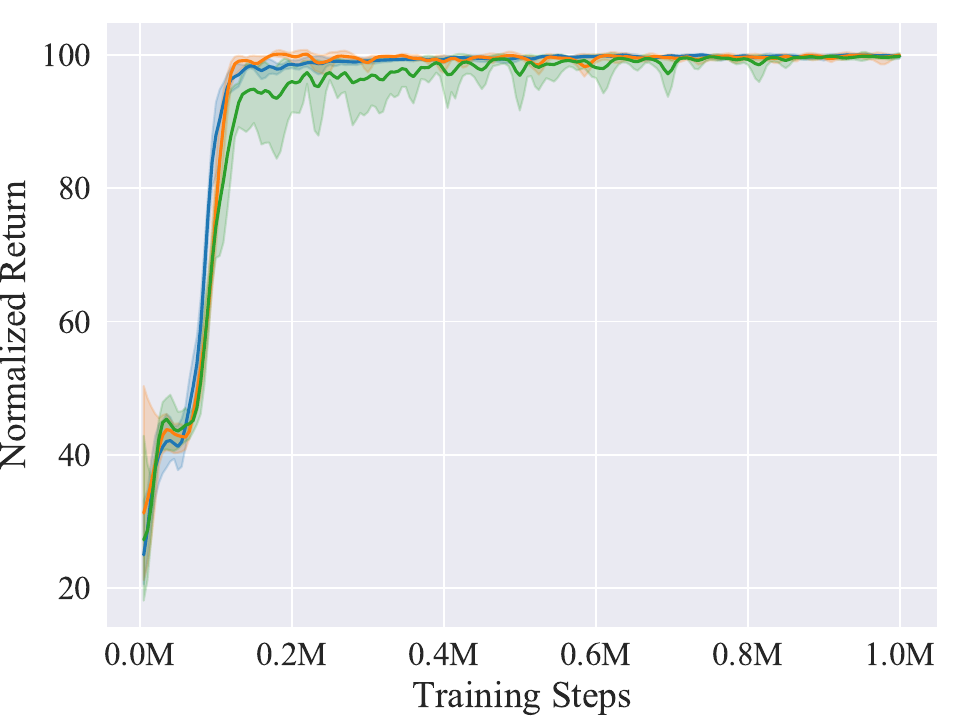}
        }\\
        \subfigure[walker2d-medium-expert-v2]{
            \includegraphics[width=0.3\linewidth]{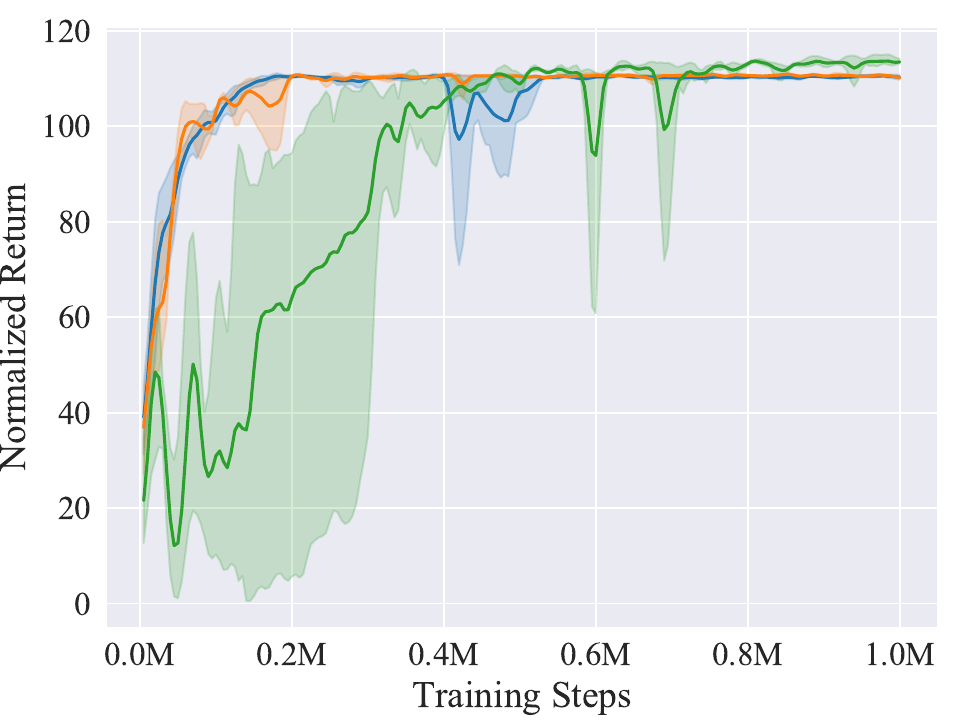}
        }
        \subfigure[halfcheetah-medium-expert-v2]{
            \includegraphics[width=0.3\linewidth]{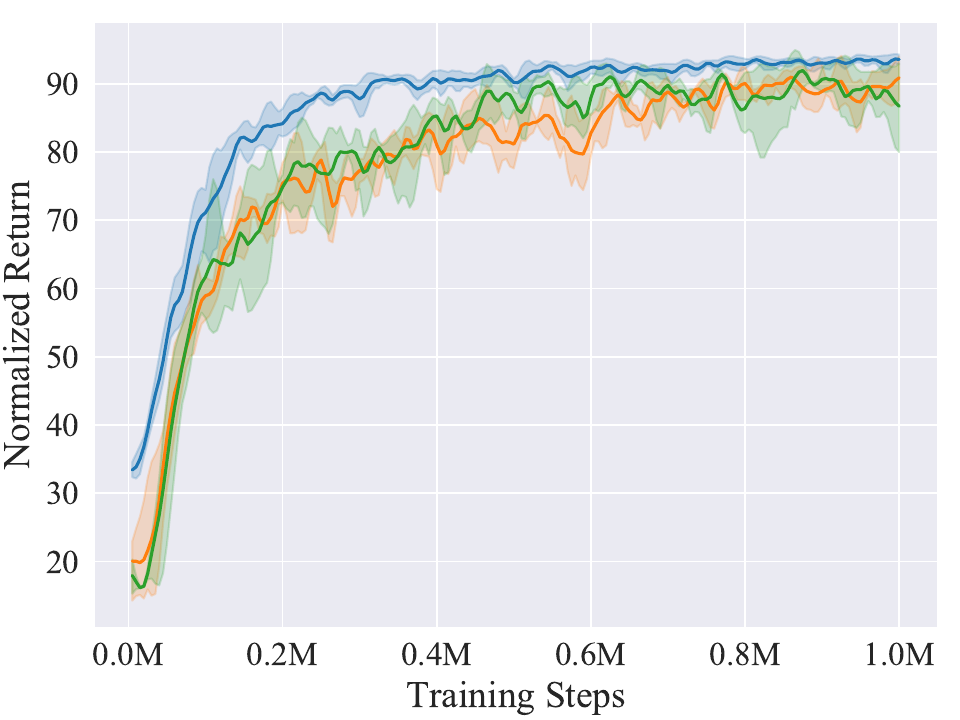}
        }
        \subfigure[hopper-medium-expert-v2]{
            \includegraphics[width=0.3\linewidth]{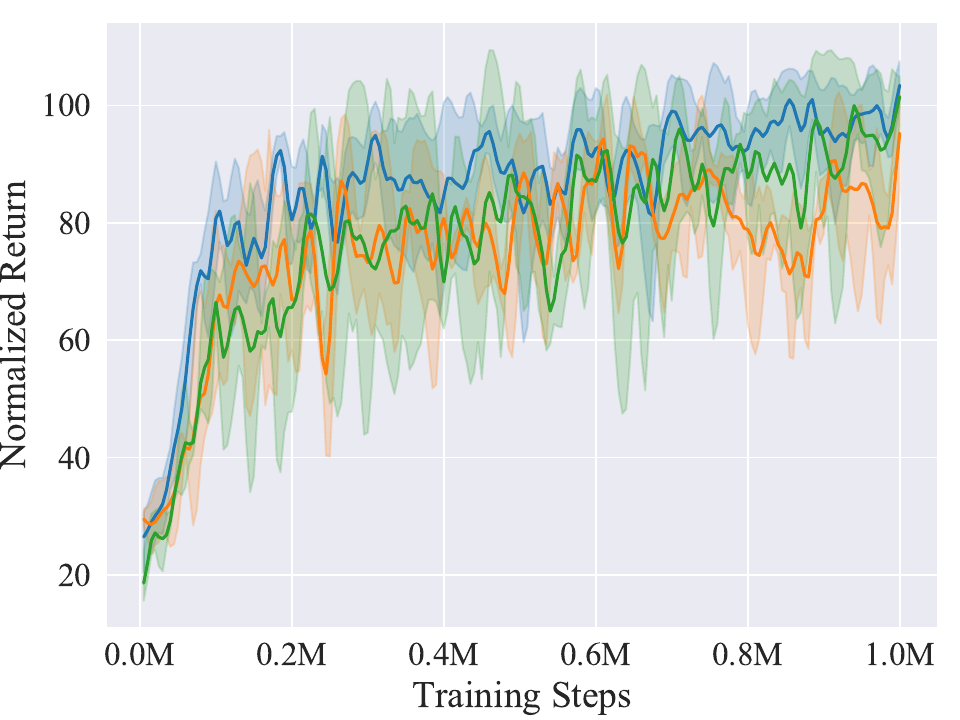}
        }
    \end{center}
    \caption{Performances of PRDC with $k$ nearest neighbors constraint.}
    \label{fig:full_ablation_k}
\end{figure*}


\end{document}